%% file: main.tex
\documentclass[referee,sn-basic]{sn-jnl}

\usepackage{graphicx}%
\usepackage{multirow}%
\usepackage{amsmath,amssymb,amsfonts}%
\usepackage{amsthm}%
\usepackage{mathrsfs}%
\usepackage[title]{appendix}%
\usepackage{xcolor}%
\usepackage{textcomp}%
\usepackage{manyfoot}%
\usepackage{booktabs}%
\usepackage{algorithm}%
\usepackage{algorithmicx}%
\usepackage{algpseudocode}%
\usepackage{listings}%
\usepackage{bm}%
\usepackage{pgfplots}%
\usepackage{mathtools}%
\usepackage{tabularx}%
\usepackage{thm-restate}%
\usepackage{xspace}%
\usepackage{placeins}%

\theoremstyle{thmstyleone}%
\newtheorem{theorem}{Theorem}
\newtheorem{proposition}[theorem]{Proposition}% 
\newtheorem{lemma}[theorem]{Lemma}

\theoremstyle{thmstyletwo}%
\newtheorem{example}{Example}%
\newtheorem{remark}{Remark}%
\newtheorem{assumption}{Assumption}%

\theoremstyle{thmstylethree}%
\newtheorem{definition}{Definition}%
\newtheorem{corollary}{Corollary}

\usepackage{macros}

\newcommand{\algo}{\textsc{DEF-PG}\xspace}

\begin{document}

\title[Sample Complexity of Variance-Reduced Policy Gradient]
{\center\Large Sample Complexity of \\ Variance-Reduced Policy Gradient}
\subtitle{Weaker Assumptions and Lower Bounds}

\author[1,2]{\fnm{Gabor} \sur{Paczolay}}

\author[3]{\fnm{Matteo} \sur{Papini}}%\email{matteo.papini@polimi.it}

\author[1]{\fnm{Alberto Maria} \sur{Metelli}}%\email{albertomaria.metelli@polimi.it}

\author[2]{\fnm{Istvan} \sur{Harmati}}%\email{harmati@iit.bme.hu}

\author[1]{\fnm{Marcello} \sur{Restelli}}%\email{marcello.restelli@polimi.it}

\affil[1]{\orgname{Politecnico di Milano}}

\affil[2]{\orgname{Budapest University of Technology and Economics}}

\affil[3]{\orgname{University of Milan}}

% \affil[4]{\orgdiv{Department}, \orgname{Organization}, \orgaddress{\street{Street}, \city{City}, \postcode{610101}, \state{State}, \country{Country}}}

\abstract{Several variance-reduced versions of REINFORCE based on importance sampling achieve an improved $O(\epsilon^{-3})$ sample complexity to find an $\epsilon$-stationary point, under an unrealistic assumption on the variance of the importance weights. In this paper, we propose the \algo (Defensive Policy Gradient) algorithm, based on defensive importance sampling, which achieves the same rate without any assumption on the variance of ordinary importance weights. We also establish lower bounds in a generalized black-box policy-optimization model that hides states and actions and permits parameter-dependent rewards. In this model, the optimal rates are $\Theta(\epsilon^{-4})$ with bounded-variance one-policy feedback and $\Theta(\epsilon^{-3})$ with mean-square-smooth coupled two-policy feedback. Under standard policy-regularity conditions, REINFORCE and \algo realize the corresponding oracle conditions and attain the $O(\epsilon^{-4})$ and $O(\epsilon^{-3})$ upper bounds, respectively. Although the lower bounds do not apply directly to the classical MDP interaction model in which these algorithms operate, this correspondence provides oracle-level evidence that the faster rate of \algo is optimal and genuinely separated from that of vanilla policy gradient. 
%Numerical simulations show promising results for the proposed technique compared to similar algorithms based on ordinary importance sampling.
}

\maketitle

\input{src/Introduction}

\input{src/Preliminaries}

\input{src/Algorithm}

\input{src/Upper_Bounds}

\input{src/Lower_Bounds}

\input{src/Related_Works}

\input{src/Conclusion}

\FloatBarrier

\backmatter

\bmhead{Acknowledgments}
Funded by the European Union – Next Generation EU within the project NRPP M4C2, Investment 1.,3 DD. 341 -  15 march 2022 – FAIR – Future Artificial Intelligence Research – Spoke 4 - PE00000013 - D53C22002380006.

\bibliography{main.bib}

\clearpage

\begin{appendices}
    \input{src/Appendix}
\end{appendices}

\end{document}

%% file: src/Introduction.tex
\section{Introduction}\label{sec:intro}

Policy gradient methods (PG) are a family of reinforcement learning algorithms~\citep{sutton2018reinforcement} that are particularly suited for large-scale decision and control problems. Recent applications of PG algorithms range from video games~\citep{wurman2022outracing}, to robotics~\citep{rudin2021learning}, to large language models~\citep{openai2023gpt4,ahmadian2024basics}, where fine-tuning of the LLM with these methods shows a concrete application within the current conversational artificial intelligence boom, and the potential impact of accelerated versions of these methods. In fact, the widespread adoption of these methods is still limited by their need for a huge amount of training data, meaning very long training times in the case of cyber-physical systems or massive computational resources in the case of virtual systems. This motivates the study of the sample complexity of PG from a theoretical perspective~\citep{agarwal2021theory} and the development of accelerated algorithms from a practical one.

The sample complexity of policy gradient algorithms can be measured in terms of the number of interaction episodes, or \emph{samples}, required to find an $\epsilon$-stationary point~\citep{papini2018stochastic}. More formally, if $\Theta\subseteq\Reals^d$ is a space of policy parameters and $J:\Theta\to\Reals$ is an objective function measuring the performance of a policy, the goal is to find a $\vtheta\in\Theta$ such that
\begin{equation*}
    \norm{\nabla_{\vtheta} J(\vtheta)} \le \epsilon,
\end{equation*}
using as few samples as possible, or at least with a number of samples that is polynomial in $\epsilon^{-1}$ and all relevant problem-dependent quantities. Alternatively, under stronger assumptions, one can study the number of samples necessary to find a globally optimal policy~\citep{liu2020improved}. However, since $J$ can be nonconvex, convergence to a local optimum is the best one can hope for in full generality. In this paper, we will focus on convergence to $\epsilon$-first-order-stationary points ($\epsilon$-FOSP). Furthermore, we will restrict our attention to policy gradient algorithms that use first-order information only, which are typically more computationally efficient~\citep{fatkhullin2023improved}.

The prototypical policy gradient algorithm is REINFORCE~\citep{williams1992simple} and its sample complexity is $O(\epsilon^{-4})$. This simply follows from the sample complexity of stochastic gradient descent~\citep{bertsekas1996neuro,sutton1999policy}, but the details were only worked out recently~\cite[see, e.g.,][]{yuan2022general}. Variants of REINFORCE based on stochastic variance reduction, a technique originally developed for finite-sum optimization~\citep{johnson2013accelerating}, can obtain a better sample complexity. The first algorithm to break the $\epsilon^{-4}$ barrier was SVRPG~\citep{papini2018stochastic}, for which~\citet{xu2019improved} proved $O(\epsilon^{-10/3})$ sample complexity. A further improvement was achieved by SRVR-PG~\citep{xu2020sample}, with a sample complexity of $O(\epsilon^{-3})$. This is commonly believed to be optimal~\citep[e.g.,][]{yuan2020stochastic} due to a lower bound for the related setting of nonconvex stochastic optimization~\citep{arjevani2023lower}, although the latter is not directly comparable with policy optimization.
Several other algorithms followed: STORM-PG~\citep{yuan2020stochastic}, PAGE-PG~\citep{gargiani2022page}, and others~\citep{pham2020hybrid,huang2020momentum}, all with $O(\epsilon^{-3})$ sample complexity. Modest empirical improvements over REINFORCE were also observed in these works.

However, all these algorithms employ importance weights to keep the variance-reduced gradient unbiased, an issue specific to RL~\citep{papini2018stochastic}, which is absent in stochastic optimization. Their sample complexity upper bounds rely on an unrealistic assumption on the variance of the importance weights, which can be infinite in practically relevant cases~\citep{zhang2021convergence}. Notable exceptions are algorithms based on second-order information~\citep{shen2019hessian,fatkhullin2023improved}, and TSIVR-PG by~\citet{zhang2021convergence}. The latter could remove the unrealistic assumption (while still using importance weights) by truncating the gradients. However, their $\widetilde{O}(\epsilon^{-3})$ upper bound only holds for softmax policy parametrizations, which are impractical for continuous action spaces. More recently~\citet{fatkhullin2023improved} proposed an accelerated policy gradient algorithm that uses neither importance weighting nor second-order information. However, its sample complexity to find an $\epsilon$-FOSP is $O(\epsilon^{-3.5})$. It is then only natural to ask the following question:

\emph{
    Can policy gradient algorithms achieve $O(\epsilon^{-3})$ sample complexity to find an $\epsilon$-FOSP, using first-order information only and without any assumption on the variance of importance weights?
}

In this paper, we answer positively by designing a variance-reduced policy gradient algorithm based on defensive importance sampling~\citep{owen2000safe}, a variant of importance weighting with naturally bounded variance. 
Although we build on existing variance-reduced policy gradient algorithms~\citep[][in particular]{gargiani2022page}, we are the first to introduce to this line of work the defensive importance sampling technique, which overcomes in a very natural way a major limitation of existing algorithms. Fully exploiting this method poses two main challenges. The first is to design and analyze a defensive sampling scheme within the context of policy gradient learning. The second, more technical, is a characterization of the \emph{continuity} of the variance of the defensive importance weights, as boundedness alone is not enough to prove the sample complexity upper bound (Lemma~\ref{lem:vardef}). 

Moreover, we establish lower bounds in a generalized black-box
policy-optimization model where the learner observes rewards and policy
scores but not states or actions, and where the reward mechanism may
depend explicitly on the queried policy parameter. In this model, the
optimal rates are $\Theta(\epsilon^{-4})$ with bounded-variance
one-policy feedback and $\Theta(\epsilon^{-3})$ with mean-square-smooth
coupled two-policy feedback. Under standard policy-regularity
conditions, REINFORCE and \algo realize the corresponding oracle
conditions and attain the $O(\epsilon^{-4})$ and $O(\epsilon^{-3})$
upper bounds, respectively. Although the lower bounds do not apply
directly to the classical observed-trajectory MDP model, this
correspondence provides oracle-level evidence that the faster rate of
\algo reflects a genuine improvement over that of REINFORCE.

%Besides answering these theoretical questions, an algorithm based on defensive importance sampling has the potential to be more stable in practice, leading to faster learning. We investigate this through numerical simulations on continuous control benchmarks.

%% file: src/Preliminaries.tex
\section{Preliminaries}\label{sec:prelim}

\paragraph{Notation.} For a measurable space $\mathcal{X}$, we denote by $\Delta_{\mathcal{X}}$ the set of probability measures over $\mathcal{X}$. We denote by $\Delta_m$ the probability simplex over $m$ outcomes. We denote by $\chi^2(P\|Q)$ the chi-square divergence between
distributions $P$ and $Q$.\linebreak

In this section, we introduce the basics of policy optimization, off-policy gradient estimation, and the key assumptions that will be used in our study of sample complexity.

\paragraph{Policy optimization. }
We consider independent episodes of interaction, of indefinite length, between an agent and an environment. In each timestep $h=0,1,\dots$ of an episode, the agent observes a state $S_h$, performs an action $A_h$, and receives a reward $R_{h+1}$.
We model this as a discounted Markov decision problem~\citep{puterman2014markov}, that is a $6$-tuple $(\mathcal{S},\mathcal{A}, P, R,\mu_0,\gamma)$. Here $\mathcal{S}$ is a state space, which is either very large or continuous; $\mathcal{A}$ is an action space, which may be discrete or continuous; $P:\mathcal{S}\times\mathcal{A}\to\Delta_{\mathcal{S}}$ is a transition probability kernel, such that $p(s'|s,a)=\mathbb{P}(S_{h+1}=s'|S_h=s,A_h=a)$ (or the corresponding probability density); $R:\mathcal{S}\times\mathcal{A}\to\Reals$ is a reward function, such that $R_{h+1}=R(S_h,A_h)$; $\gamma\in[0,1)$ is a discount factor; and $\mu_0\in\Delta_{\mathcal{S}}$ is the initial state distribution, such that $\mu_0(s)=\mathbb{P}(S_0=s)$ (or the corresponding probability density). 

We consider an episodic learning process in which the agent commits, at
the beginning of each episode, to a Markovian stationary stochastic
policy $\pi:\mathcal S\to\Delta_{\mathcal A}$, with
$\pi(a\mid s)=\mathbb P(A_h=a\mid S_h=s)$, or the corresponding
probability density. The goal is to maximize the expected discounted
return:
\begin{equation*}
    \max_{\pi} J(\pi) \coloneqq \EE{\sum_{h\ge 0}\gamma^hR_{h+1}\,\Bigg|\, \pi},
\end{equation*}
where the expectation is with respect to $S_0\sim\mu_0$, $A_h\sim\pi(\cdot|S_h)$, and $S_{h+1}\sim P(\cdot|S_h,A_h)$.

In particular, we restrict our attention to parametric policies with
parameter space $\Theta=\Reals^d$. The corresponding policy space is
$\Pi_{\Theta}=\{\pi_{\vtheta}\mid\vtheta\in\Theta\}$, where each policy
is parametrized by an element of $\Theta$. 
Throughout, we tacitly assume the standard regularity conditions for likelihood-ratio policy-gradient methods~\citep{sutton1999policy}, including parameter-independent common support and conditions permitting differentiation under the relevant trajectory integrals and discounted sums.
The policy-optimization
problem can then be written, with slight abuse of notation, as
\begin{equation*}
    \max_{\vtheta\in\Theta} J(\vtheta) \coloneqq  J(\pi_{\vtheta}).
\end{equation*}

In the following, $J$ will always denote the parametric objective function, that is,~$J:\Theta\to\Reals$. We call it the \emph{performance function}. Whenever $|R(s,a)|\le R_{\max}$, it satisfies~$|J(\vtheta)|\leq R_{\max}/(1-\gamma)$ for all $\vtheta\in\Theta$.

In general, $J$ is nonconvex, but it has the structure of an expectation under the trajectory distribution induced by the policy.

\begin{remark}[Geometric stopping]
\label{rem:geometric-stopping}
For the estimators, we use an equivalent finite-trajectory representation of the discounted objective. Independently of the policy and the environment, terminate each episode after every transition with probability $1-\gamma$. If $H_\gamma$ denotes the resulting trajectory length, then
\[
\mathbb P(H_\gamma>h)=\gamma^h,
\qquad
\mathbb E[H_\gamma]=\frac{1}{1-\gamma}.
\]
We henceforth let $p(\cdot\mid\vtheta)$ denote the distribution of the stopped trajectory $\tau=(S_0,A_0,R_1,\ldots,S_{H_\gamma-1},A_{H_\gamma-1},
R_{H_\gamma},S_{H_\gamma})$, and define its undiscounted stopped return by
\[
R(\tau)=\sum_{h=0}^{H_\gamma-1}R_{h+1}.
\]
Since the stopping rule is independent of the trajectory,
\[
\mathbb E_{\tau\sim p(\cdot\mid\vtheta)}[R(\tau)]
=
\mathbb E\left[\sum_{h\geq 0}\gamma^hR_{h+1}
\,\middle|\,\pi_{\vtheta}\right]
=
J(\vtheta).
\]
Thus, geometric stopping introduces no approximation or truncation bias. It also introduces no additional dependence on $\epsilon$. If individual environment transitions are counted instead of trajectories, each oracle call has expected cost $(1-\gamma)^{-1}$.
\end{remark}

The performance function can therefore be written as
\begin{equation*}
J(\vtheta)
=
\mathbb E_{\tau\sim p(\cdot\mid\vtheta)}[R(\tau)].
\end{equation*}
The stopped return is a fixed function of the trajectory and depends on $\vtheta$ only through the trajectory distribution $p(\cdot\mid\vtheta)$. Both $R$ and the mapping $\vtheta\mapsto p(\cdot\mid\vtheta)$ are unknown to the agent.

We are interested in algorithms that run a sequence of policies $\pi_{\vtheta_1},\pi_{\vtheta_2},\dots$, each for one or more independent episodes, collecting datasets of trajectories $\mathcal{D}_1,\mathcal{D}_2,\dots$, where $\mathcal{D}_k=\{\tau^k_1,\dots,\tau^k_{n_k}\}$ is a dataset of $n_k$ \emph{independent} trajectories such that $\tau_i^k\sim p(\cdot|\vtheta_k)$, with the aim of maximizing the unknown performance function $J$ from the collected data. Policy gradient algorithms~\citep{williams1992simple} do this via stochastic gradient ascent:
\begin{equation*}
    \vtheta_{k+1} = \vtheta_k + \eta \widehat{\nabla} J(\vtheta_k),
\end{equation*}
where $\eta>0$ is a step size and $\widehat{\nabla}J(\vtheta)$ is an estimate of the gradient $\nabla_{\vtheta}J(\vtheta)$ of the performance function. In the following, we will omit the $\vtheta$ subscript since gradients are always with respect to the policy parameters. The simplest policy gradient algorithm is REINFORCE~\citep{williams1992simple}, which we present here in its refined version for Markov decision processes, known as GPOMDP~\citep{baxter2001infinite}. Here $\widehat{\nabla}J(\vtheta_k)$ is computed from the batch of trajectories $\mathcal{D}_k$ collected with $\pi_{\vtheta_k}$ as follows

\begin{equation}\label{eq:gpomdp}
\begin{aligned}
\widehat{\nabla}J(\vtheta_k)
    &=
    \frac{1}{n_k}
    \sum_{i=1}^{n_k}g(\vtheta_k,\tau_i),\\
g(\vtheta,\tau)
    &=
    \sum_{h=0}^{H_\gamma-1}
    \left(
        \sum_{\ell=0}^{h}
        \nabla\log\pi_{\vtheta}(A_\ell\mid S_\ell)
    \right)R_{h+1},
\end{aligned}
\end{equation}

where the mapping $\vtheta\mapsto \pi_{\vtheta}$ is assumed to be differentiable with respect to $\vtheta$. The GPOMDP gradient estimator is unbiased, that is, $\EEp{\tau\sim p(\cdot|\vtheta)}{g(\vtheta, \tau)} = \nabla J(\vtheta)$, and the same holds for $\widehat{\nabla}J(\vtheta)$ since the trajectories are independent. Note how, for this to be true, the sample trajectory must be collected with the same policy parameters of which we are estimating the gradient. This is \emph{on-policy} gradient estimation. Moreover, under standard assumptions on the policy class, the variance of the GPOMDP estimator is bounded~\citep{zhao2012analysis,papini2022smoothing}, albeit infamously large. Common refinements of GPOMDP are variance-reducing baselines~\citep{peters2008reinforcement} and actor-critic algorithms~\citep{barto1983neuronlike,konda1999actor,sutton1999policy}, but we will stick to the simple expression from Equation~\eqref{eq:gpomdp} in our theoretical treatment.

\paragraph{Off-policy policy gradient. }
Off-policy gradient estimation is also possible~\citep{degris2012linear}. If $\tau\sim p(\cdot|\vtheta')$ is collected with a different policy $\pi_{\vtheta'}$, an unbiased estimator of $\nabla J(\vtheta)$ is:
\begin{equation*}
    \widetilde{g}(\vtheta,\tau) = \frac{p(\tau|\vtheta)}{p(\tau|\vtheta')}g(\vtheta,\tau).
\end{equation*}
Note that $g(\vtheta,\tau)$ is \emph{not} an unbiased estimator for $\nabla J(\vtheta)$ in this case since $\tau$ is collected with a different policy $\pi_{\vtheta'}$. However, it is easy to see that the ratio $p(\tau|\vtheta)/p(\tau|\vtheta')$ ensures $\EEp{\tau\sim p(\vtheta')}{\widetilde{g}(\vtheta,\tau)}=\nabla J(\vtheta)$. This is called \emph{importance weighting}. The importance weight is usually easy to compute: since the reward $R_{h+1}$ and the next-state $S_{h+1}$ are independent of $\vtheta$ conditionally on the action $A_h$, the importance weight for $\tau\sim p(\cdot|\vtheta')$ is simply:\footnote{The geometric stopping probabilities also cancel from this ratio because they are independent of the policy parameters.}
\begin{equation*}
\frac{p(\tau\mid\vtheta)}
     {p(\tau\mid\vtheta')}
=
\prod_{h=0}^{H_\gamma-1}
\frac{\pi_{\vtheta}(A_h\mid S_h)}
     {\pi_{\vtheta'}(A_h\mid S_h)}
\eqqcolon
\omega(\vtheta,\tau).
\end{equation*}

Of course, when $\vtheta=\vtheta'$, $\omega(\vtheta,\tau)=1$ for any $\tau\sim p(\cdot|\vtheta)$ and we recover the on-policy case.
Unfortunately, the variance of the off-policy estimator is not bounded in general since the variance of importance weights can be infinite even for simple policy classes~\citep{metelli2020importance,zhang2021convergence}. It is still possible to characterize the variance of importance weights in terms of the chi-square divergence~\citep{cortes2010learning}:
\begin{equation*}
    \VVp{\tau\sim p(\cdot|\vtheta')}{\omega(\vtheta,\tau)} = \int_{\mathcal{T}} \frac{p(\tau|\vtheta)^2}{p(\tau|\vtheta')}\de\tau - 1 = \chi^2(p(\cdot|\vtheta)\|p(\cdot|\vtheta')),
\end{equation*}
where $\mathcal{T}$ denotes the set of feasible trajectories. However, the chi-square divergence is unbounded in general, and it is even larger than the Kullback--Leibler divergence that appears more frequently in the policy optimization literature~\citep{schulman2015trust}. For this reason, most works on the sample complexity of policy gradient~\citep{papini2018stochastic,xu2020sample,yuan2020stochastic,gargiani2022page} resort to the following strong assumption, a uniform upper bound on the variance of the importance weights:

\setcounter{assumption}{-1}
\begin{assumption}\label{asm:var_iw}
    There exists a finite constant $W>0$ such that
    \begin{equation*}
        \sup_{\vtheta_1,\vtheta_2\in\Theta}\VVp{\tau\sim p(\cdot|\vtheta_2)}{\omega(\vtheta_1,\tau)} \le W.
    \end{equation*}
\end{assumption}
As discussed by~\citet{zhang2021convergence}, this assumption can fail
even for simple policy classes, as illustrated by
Example~\ref{ex:unbounded-iw} below.

\begin{example}\label{ex:unbounded-iw}
    Consider a one-step MDP and a class of one-dimensional Gaussian policies with constant mean, where $\vtheta\in\Reals$ parametrizes the standard deviation as $\sigma_{\vtheta}=e^{\vtheta}$. For $\theta_1\neq\theta_2$, using the relation between the chi-square
divergence and the variance of the importance weights together with
the closed-form expression for Gaussian distributions~\citep{gil2013renyi},
we obtain
    \begin{align*}
        \VVp{\tau\sim p(\cdot|\theta_2)}{\omega(\theta_1,\tau)} = \frac{\sigma_{\theta_2}^2}{\sigma_{\theta_1}\sqrt{2\sigma_{\theta_2}^2-\sigma_{\theta_1}^2}}-1,
    \end{align*}
    when $2\sigma_{\theta_2}^2-\sigma_{\theta_1}^2>0$; otherwise, the variance is infinite. Thus, it suffices to select the target parameter $\theta_1$ such that:
    \begin{align*}
        2e^{2\theta_2}-e^{2\theta_1} \le 0 \implies \theta_1 \ge   \theta_2 + \log \sqrt{2},
    \end{align*}
    to make the variance infinite.  We remark that this result holds for every $\theta_2 \in \mathbb{R}$, i.e., we do not need to send the Gaussian standard deviation to zero to make the variance of the importance weights diverge.
\end{example}

\paragraph{Policy-gradient oracle models. }
We now abstract away from the particular construction of policy-gradient estimators and isolate the properties of stochastic first-order feedback that determine the attainable sample-complexity rate. We consider two forms of feedback for maximizing a smooth performance function: a bounded-variance one-policy oracle, which captures vanilla policy gradient, and a coupled two-policy oracle, which captures the correlation exploited by variance-reduced methods. Smoothness of the objective is characterized as follows.

\begin{assumption}[Smooth performance]
\label{asm:smooth-performance}
The performance function is $L$-smooth, that is, for all
$\vtheta,\vtheta'\in\Theta$,
\[
\left\|\nabla J(\vtheta)-\nabla J(\vtheta')\right\|
\leq L\left\|\vtheta-\vtheta'\right\|.
\]
\end{assumption}

The one-policy oracle returns an unbiased gradient estimate at a single queried policy.

\begin{assumption}[Bounded-variance one-policy oracle]
\label{asm:bv-oracle}
When queried at $\vtheta$, the oracle $\mathcal O_{\mathrm{BV}}$ returns a random vector
$\widehat g(\vtheta;\zeta)$ satisfying
\[
\mathbb E\left[
\widehat g(\vtheta;\zeta)
\,\middle|\,\vtheta
\right]
=\nabla J(\vtheta)
\]
and
\[
\mathbb E\left[
\left\|
\widehat g(\vtheta;\zeta)-\nabla J(\vtheta)
\right\|^2
\,\middle|\,\vtheta
\right]
\leq \nu^2.
\]
Different oracle calls use independent internal randomness conditionally on the queried parameters.
\end{assumption}

Variance-reduced methods require stronger feedback: two unbiased estimates sharing randomness whose difference becomes small when the queried parameters are close. This property need not hold for two independent calls to $\mathcal O_{\mathrm{BV}}$.

\begin{assumption}[Mean-square-smooth coupled two-policy oracle]
\label{asm:mss-oracle}
In addition to the one-policy oracle of Assumption~\ref{asm:bv-oracle}, the learner can query a pair
$(\vtheta,\vtheta')$. The oracle $\mathcal O_{\mathrm{MSS}}$ returns two coupled random vectors
\[
\left(
\widehat g_1(\vtheta,\vtheta';\zeta),
\widehat g_2(\vtheta,\vtheta';\zeta)
\right),
\]
such that
\[
\mathbb E\left[
\widehat g_1(\vtheta,\vtheta';\zeta)
\,\middle|\,\vtheta,\vtheta'
\right]
=\nabla J(\vtheta),
\qquad
\mathbb E\left[
\widehat g_2(\vtheta,\vtheta';\zeta)
\,\middle|\,\vtheta,\vtheta'
\right]
=\nabla J(\vtheta'),
\]
and
\[
\mathbb E\left[
\left\|
\widehat g_1(\vtheta,\vtheta';\zeta)
-
\widehat g_2(\vtheta,\vtheta';\zeta)
\right\|^2
\,\middle|\,\vtheta,\vtheta'
\right]
\leq
L_{\mathrm{pair}}^2
\left\|\vtheta-\vtheta'\right\|^2.
\]
The two outputs within a coupled call need not be independent. Different coupled calls use independent internal randomness conditionally on the queried parameters.
\end{assumption}

We count one call to $\mathcal O_{\mathrm{BV}}$ as one one-policy oracle call and one call to $\mathcal O_{\mathrm{MSS}}$ as one coupled two-policy oracle call. In the defensive importance-sampling realization introduced below, a single trajectory produces both coupled outputs, so one call of either type costs one trajectory.

%% file: src/Algorithm.tex
\section{Algorithm}\label{sec:algo}
In this section, we present our modified policy gradient algorithm. Like SVRPG~\citep{papini2018stochastic} and its successors~\citep{xu2020sample,yuan2020stochastic,gargiani2022page}, our algorithm is based on the technique of stochastic variance reduction from finite-sum optimization~\citep{leroux2012stochastic,johnson2013accelerating}.
The idea of stochastic variance-reduction is to alternate between ``full gradient'' and ``stochastic gradient'' updates. In the context of machine learning~\citep{lei2017less}, a full gradient is a gradient estimate computed from a large batch of data, and a stochastic gradient is computed from a smaller batch or a single sample.
The full gradient estimate is more precise and allows more stable updates, but has a larger impact on the sample complexity. The key idea is to combine infrequent full gradient estimates with frequent stochastic gradients, using the former to stabilize the latter. 

As a concrete example, consider the modified gradient estimator used by the PAGE algorithm~\citep{li2021page}:
\begin{equation}\label{eq:page}
    v_t = 
        \begin{cases}
            \frac{1}{N}\sum_{i=1}^N g(x_{t},z_i) & \text{with probability } p \\
            \frac{1}{B}\sum_{i=1}^B g(x_{t},z_i) + v_{t-1} - \frac{1}{B}\sum_{i=1}^B g(x_{t-1},z_i)& \text{with probability } 1-p,
        \end{cases}
\end{equation}
where the $z_i$ are i.i.d. samples from a fixed data distribution $\rho$, the objective function is $F(x)=\EEp{z \sim \rho}{f(x,z)}$, $g$ is an unbiased gradient estimator, $\EEp{z\sim \rho}{g(x,z)}=\nabla_x F(x)$, $B\ll N$, $0<p\ll 1$, and $v_0=\frac{1}{N}\sum_{i=1}^N g(x_0,z_i)$, using an independent initial batch. 
Since $p$ is small, the full gradient is computed infrequently. In most iterations, instead, it is combined, in a recursive fashion, with a stochastic gradient computed from a smaller number of samples.
The negative term in Equation~\eqref{eq:page} is a \emph{correction term}, which is fundamental to control the bias-variance tradeoff.

Applying this and similar techniques to policy optimization comes with an additional challenge that is peculiar to RL, as observed by~\citet{papini2018stochastic}: \emph{distribution shift}. Here, the trajectories are not sampled from a fixed distribution, but from $p(\cdot|\vtheta_t)$, the induced trajectory distribution, which changes with time and is itself a function of the optimization variable $\vtheta$. So far, this issue has been bypassed by applying importance weighting to the correction term, at the cost of unrealistically assuming a uniform upper bound (Assumption~\ref{asm:var_iw}) on the variance of the importance weights~\citep{zhang2021convergence}.
For example, the modified policy gradient estimator for PAGE-PG~\citep{gargiani2022page} is:
\begin{equation*}
    v_{t} = 
        \begin{cases}
            \frac{1}{N}\sum_{i=1}^N g(\vtheta_t, \tau_i) & \text{prob. } p \\
            \frac{1}{B}\sum_{i=1}^B g(\vtheta_t,\tau_i) + v_{t-1} - \frac{1}{B}\sum_{i=1}^B \omega(\vtheta_{t-1},\tau_i)g(\vtheta_{t-1},\tau_i) & \text{prob. } 1-p,
        \end{cases}
\end{equation*}
where $g$ is the REINFORCE or the GPOMDP estimator and the trajectories $\tau_i$ are always sampled from $p(\cdot|\vtheta_t)$ by running the current policy $\pi_{\vtheta_t}$. Note how the importance weight is needed because the correction term is an off-policy estimate of the policy gradient of $\vtheta_{t-1}$ using trajectories from a different policy $\vtheta_t$. We refer to the former as the \emph{target policy} and to the latter as the \emph{behavior policy}.

Our key observation to avoid Assumption~\ref{asm:var_iw} is the following: although this use of importance weighting is common in offline RL~\citep{levine2020offline}, where data from the target policy is not accessible, we find ourselves in a situation where obtaining data from the target policy is actually very easy. Indeed, the target policy with parameters $\vtheta_{t-1}$ is just an older policy we ran in the previous iteration.
This allows us to use \emph{defensive importance sampling}~\citep{owen2000safe}, a special kind of mixture importance sampling that guarantees that the variance of the importance weights is always bounded.

In mixture importance sampling, to estimate $\overline{f}=\EEp{z\sim p}{f(z)}$ for $p\in\Delta_{\mathcal{Z}}$ and $f:\mathcal{Z}\to\Reals$, we draw $n$ samples $Z_1,\dots,Z_n$ from the mixture distribution $q_{\valpha} = \sum_{j=1}^m \alpha_j q_j$, where $\{q_1,\dots,q_m\}$ is a set of $m$ distributions in $\Delta_{\mathcal{Z}}$, and $\valpha\in\Delta_{m}$ is a vector of convex-combination coefficients. The mean $\overline{f}$ is estimated as:
\begin{equation*}
    \hat{f} = \frac{1}{n}\sum_{i=1}^n\frac{p(Z_i)}{\sum_{j=1}^m \alpha_j q_j(Z_i)} f(Z_i).
\end{equation*}
It is easy to see that this estimate is unbiased.
Defensive importance sampling is a special case where the mixture distribution is simply:
\begin{equation*}
    q_\alpha(x) = \alpha p(x) + (1-\alpha) q(x),
\end{equation*}
for some scalar $\alpha\in[0,1]$. Note that it recovers vanilla importance sampling for $\alpha=0$, and Monte Carlo estimation for $\alpha=1$. Most importantly, for every $\alpha\in(0,1]$,
\begin{equation*}
    \VVp{z\sim q_\alpha}{\frac{p(z)}{q_\alpha(z)}} \le \frac{1-\alpha}{\alpha}.
\end{equation*}
We will refer to $\alpha$ as the \emph{defensive} parameter: the larger the $\alpha$, the smaller the variance of the importance weights $p(z)/q_{\alpha}(z)$, and call the estimator ``{$\alpha$-defensive}''. 

\begin{algorithm}[t]
\caption{\textsc{Coupled-PAGE}}
\label{algo:coupled-page}
\begin{algorithmic}[1]
\State \textbf{Input:} initial parameter $\vtheta_0$, large batch size $N$, small batch size $B$, step size $\eta>0$, refresh probability $p\in(0,1)$, number of iterations $T$
\For{$i=1,\ldots,N$}
    \State Query $\mathcal O_{\mathrm{BV}}$ at $\vtheta_0$ and receive $\widehat g_0^{\,i}$
\EndFor
\State $v_0\gets\frac{1}{N}\sum_{i=1}^N\widehat g_0^{\,i}$
\For{$t=0,\ldots,T-1$}
    \State $\vtheta_{t+1}\gets\vtheta_t+\eta v_t$
    \If{a Bernoulli$(p)$ draw equals $1$}
        \Comment{Large-batch refresh}
        \For{$i=1,\ldots,N$}
            \State Query $\mathcal O_{\mathrm{BV}}$ at $\vtheta_{t+1}$ and receive $\widehat g_{t+1}^{\,i}$
        \EndFor
        \State $v_{t+1}\gets\frac{1}{N}\sum_{i=1}^N\widehat g_{t+1}^{\,i}$
    \Else
        \Comment{Coupled recursive update}
        \For{$i=1,\ldots,B$}
            \State Query $\mathcal O_{\mathrm{MSS}}$ at $(\vtheta_{t+1},\vtheta_t)$ and receive
            $(\widehat g_{t+1}^{\,i},\widehat g_t^{\,i})$
        \EndFor
        \State $v_{t+1}\gets
        v_t+
        \frac{1}{B}\sum_{i=1}^B
        \left(
            \widehat g_{t+1}^{\,i}
            -
            \widehat g_t^{\,i}
        \right)$
    \EndIf
\EndFor
\State \textbf{Output:} $\vtheta_{\mathrm{OUT}}$ chosen uniformly from
$\{\vtheta_t\}_{t=0}^{T-1}$
\end{algorithmic}
\end{algorithm}

\begin{algorithm}
\caption{DEF-PG}
\label{algo_ppgdef}
\begin{algorithmic}[1]
\State \textbf{Input:} initial parameter $\vtheta_0$, large batch size $N$, small batch size $B$, step size $\eta>0$, refresh probability $p\in(0,1)$, defensive parameter $\alpha\in(0,1)$, number of iterations $T$
\For{$i=1,\ldots,N$}
    \State Sample $\overline{\tau}_i\sim p(\cdot\mid\vtheta_0)$
\EndFor
\State $v_0\gets
\frac{1}{N}\sum_{i=1}^N
g(\vtheta_0,\overline{\tau}_i)$
\For{$t=0,\ldots,T-1$}
    \State $\vtheta_{t+1}\gets\vtheta_t+\eta v_t$
    \If{a Bernoulli$(p)$ draw equals $1$}
        \Comment{Large-batch refresh}
        \For{$i=1,\ldots,N$}
            \State Sample
            $\overline{\tau}_i\sim
            p(\cdot\mid\vtheta_{t+1})$
        \EndFor
        \State $v_{t+1}\gets
        \frac{1}{N}\sum_{i=1}^N
        g(\vtheta_{t+1},\overline{\tau}_i)$
    \Else
        \Comment{Defensive recursive update}
        \For{$i=1,\ldots,B$}
            \State Draw $C_i\sim\operatorname{Bernoulli}(\alpha)$
            \If{$C_i=1$}
                \State Sample
                $\tau_i\sim
                p(\cdot\mid\vtheta_{t+1})$
            \Else
                \State Sample
                $\tau_i\sim
                p(\cdot\mid\vtheta_t)$
            \EndIf
            \State $\omega_{t+1,i}\gets
            \omega_\alpha
            (\vtheta_{t+1},\vtheta_t;\tau_i)$ \Comment{Equation~\eqref{eq:defensive-weight-current}}
            \State $\omega_{t,i}\gets
            \omega_\alpha'
            (\vtheta_{t+1},\vtheta_t;\tau_i)$ \Comment{Equation~\eqref{eq:defensive-weight-previous}}
        \EndFor
        \State $v_{t+1}\gets
        v_t+
        \frac{1}{B}\sum_{i=1}^B
        \left[
            \omega_{t+1,i}
            g(\vtheta_{t+1},\tau_i)
            -
            \omega_{t,i}
            g(\vtheta_t,\tau_i)
        \right]$
    \EndIf
\EndFor
\State \textbf{Output:} $\vtheta_{\mathrm{OUT}}$ chosen uniformly from
$\{\vtheta_t\}_{t=0}^{T-1}$
\end{algorithmic}
\end{algorithm}

To separate the PAGE recursion from the construction of its policy-gradient estimates, we proceed in two steps. First, we formulate an oracle-level algorithm, \textsc{Coupled-PAGE} (Algorithm~\ref{algo:coupled-page}): large-batch refreshes use the bounded-variance one-policy oracle $\mathcal O_{\mathrm{BV}}$, while recursive updates use the coupled two-policy oracle $\mathcal O_{\mathrm{MSS}}$. We then instantiate these oracles using the GPOMDP estimator and defensive importance sampling, obtaining \algo (Algorithm~\ref{algo_ppgdef}).
A call to $\mathcal O_{\mathrm{BV}}$ at $\vtheta$ samples a stopped trajectory $\tau\sim p(\cdot\mid\vtheta)$ and returns the GPOMDP estimate $g(\vtheta,\tau)$.
For a coupled call at $(\vtheta,\vtheta')$, define
$
q_{\alpha,\vtheta,\vtheta'}(\tau)
\coloneqq
\alpha p(\tau\mid\vtheta)
+
(1-\alpha)p(\tau\mid\vtheta')$,
and the defensive weights
\begin{align}
\omega_\alpha(\vtheta,\vtheta';\tau)
&\coloneqq
\frac{p(\tau\mid\vtheta)}
     {q_{\alpha,\vtheta,\vtheta'}(\tau)}
=
\frac{1}{
    \alpha+
    (1-\alpha)
    \frac{p(\tau\mid\vtheta')}
         {p(\tau\mid\vtheta)}
},
\label{eq:defensive-weight-current}
\\
\omega_\alpha'(\vtheta,\vtheta';\tau)
&\coloneqq
\frac{p(\tau\mid\vtheta')}
     {q_{\alpha,\vtheta,\vtheta'}(\tau)}
=
\frac{1}{
    \alpha
    \frac{p(\tau\mid\vtheta)}
         {p(\tau\mid\vtheta')}
    +1-\alpha
}.
\label{eq:defensive-weight-previous}
\end{align}
A trajectory from $q_{\alpha,\vtheta,\vtheta'}$ is generated by running
$\pi_{\vtheta}$ with probability $\alpha$ and $\pi_{\vtheta'}$ with probability $1-\alpha$. The likelihood ratios in
Equations~\eqref{eq:defensive-weight-current}--\eqref{eq:defensive-weight-previous}
are finite products of policy ratios by Remark~\ref{rem:geometric-stopping}.

Algorithms~\ref{algo:coupled-page} and~\ref{algo_ppgdef} have the same recursion. In DEF-PG, every on-policy GPOMDP sample realizes one call to $\mathcal O_{\mathrm{BV}}$, while every trajectory sampled from the defensive mixture produces the coupled pair
$\left(
    \omega_{t+1,i}g(\vtheta_{t+1},\tau_i),
    \omega_{t,i}g(\vtheta_t,\tau_i)
\right)$.
Proposition~\ref{prop:gpomdp_var} verifies Assumption~\ref{asm:bv-oracle} for the first type of call, and Lemma~\ref{lem:vardef} verifies Assumption~\ref{asm:mss-oracle} for the second. Consequently, DEF-PG is a concrete implementation of \textsc{Coupled-PAGE}, with one stopped trajectory per oracle call.

%% file: src/Upper_Bounds.tex
\section{Sample Complexity Upper Bounds}\label{sec:upper}
In this section, we first recall the $O(\epsilon^{-4})$
sample-complexity upper bound under the smoothness and bounded-variance
one-policy oracle assumptions, and then establish an
$O(\epsilon^{-3})$ upper bound for \textsc{Coupled-PAGE} under the
coupled oracle assumptions. We subsequently give concrete conditions
under which the GPOMDP estimator realizes the one-policy oracle and defensive
importance sampling realizes the coupled oracle, yielding the
corresponding guarantees for GPOMDP and \algo. Throughout this section,
let
$J^\star\coloneqq\sup_{\vtheta\in\Theta}J(\vtheta)$ and
$\Delta\coloneqq J^\star-J(\vtheta_0)$, and assume that
$\Delta<+\infty$.

\paragraph{Vanilla policy gradient.}
We begin with a standard oracle-level upper bound for stochastic gradient
ascent. The result closely parallels the policy-gradient analysis of
\citet[Corollary~4.7]{yuan2022general}. We include a short proof in Appendix \ref{app:proofprop} since
our formulation uses the bounded-variance oracle directly and does not
involve a finite-horizon approximation.%%%
\begin{proposition}[One-policy oracle upper bound]
\label{prop:one-policy-upper}
Suppose that Assumptions~\ref{asm:smooth-performance} and~\ref{asm:bv-oracle} hold, and let
$\Delta
\coloneqq
J^\star -J(\vtheta_0)
<+\infty$.
Then mini-batch stochastic gradient ascent using independent calls to
$\mathcal O_{\mathrm{BV}}$ finds a parameter
$\vtheta_{\mathrm{OUT}}$ satisfying
$\mathbb E[\|\nabla J(\vtheta_{\mathrm{OUT}})\|]\leq\epsilon$
using
\[
O\left(
    \frac{\Delta L}{\epsilon^2}
    +
    \frac{\Delta L\nu^2}{\epsilon^4}
\right)
\]
one-policy oracle calls.
\end{proposition}

We then provide two standard conditions that allow us to realize the bounded-variance one-policy oracle using the GPOMDP estimator.

\begin{assumption}[Bounded rewards]
\label{asm:bounded-rewards}
The per-step rewards are uniformly bounded: there exists $R_{\max}<+\infty$ such that
$
|R(s,a)|\leq R_{\max}
$
for all $s\in\mathcal S$ and $a\in\mathcal A$.
\end{assumption}

\begin{assumption}[Bounded score second moment]
\label{asm:score-moment}
The policy class admits a finite constant $c_1>0$ such that, for every
$\vtheta\in\Theta$,
\[
\sup_{s\in\mathcal S}
\mathbb E_{a\sim\pi_{\vtheta}(\cdot\mid s)}
\left[
    \left\|
        \nabla\log\pi_{\vtheta}(a\mid s)
    \right\|^2
\right]
\leq c_1^2.
\]
\end{assumption}

At the oracle level, the $O(\epsilon^{-4})$ upper bound only requires Assumptions~\ref{asm:smooth-performance} and~\ref{asm:bv-oracle}, together with a finite initial gap. Under Assumptions~\ref{asm:bounded-rewards} and~\ref{asm:score-moment}, the GPOMDP estimator satisfies this oracle condition.

\begin{proposition}[Variance of the geometrically stopped GPOMDP estimator, cf. \citealp{papini2022smoothing}, Lemma 29. Proof provided in Appendix \ref{app:proofprop}.]
\label{prop:gpomdp_var}
Under Assumptions~\ref{asm:bounded-rewards} and~\ref{asm:score-moment}, for every $\vtheta\in\Theta$,
\[
\mathbb E_{\tau\sim p(\cdot\mid\vtheta)}
\left[
    \left\|
        g(\vtheta,\tau)-\nabla J(\vtheta)
    \right\|^2
\right]
\leq
\frac{8c_1^2R_{\max}^2}{(1-\gamma)^3}
\eqqcolon \nu^2.
\]
\end{proposition}

Consequently, the GPOMDP estimator defines an instance of the one-policy oracle $\mathcal O_{\mathrm{BV}}$, with variance parameter $\nu^2$ as given above.
Then, the sample complexity of GPOMDP can be bounded as follows.

\begin{proposition}[Sample complexity of GPOMDP]
\label{prop:gpomdp_rate}
Under Assumptions~\ref{asm:smooth-performance},
\ref{asm:bounded-rewards}, and~\ref{asm:score-moment}, the following
total number of trajectories is sufficient for GPOMDP to output an
$\epsilon$-FOSP in expectation, that is, $\mathbb E[\|\nabla J(\vtheta_{\mathrm{OUT}})\|]\leq\epsilon$:
\[
N_{\mathrm{TOT}}
=
O\left(
    \frac{\Delta L}{\epsilon^2}
    +
    \frac{\Delta L\nu^2}{\epsilon^4}
\right)
=
O(\epsilon^{-4}),
\]
where $\nu^2$ is the constant from
Proposition~\ref{prop:gpomdp_var}.
\end{proposition}

\begin{proof}
Proposition~\ref{prop:gpomdp_var} shows that the geometrically stopped
GPOMDP estimator satisfies Assumption~\ref{asm:bv-oracle} with variance
parameter $\nu^2$. The result therefore follows directly from
Proposition~\ref{prop:one-policy-upper}.
\end{proof}
This recovers, for the geometrically stopped formulation, the standard
$O(\epsilon^{-4})$ dependence established for vanilla policy gradient
by~\citet[Corollary~4.7]{yuan2022general}.
If one additionally has $L=O((1-\gamma)^{-2})$, the preceding
result yields
$O(\Delta(1-\gamma)^{-2}\epsilon^{-2}
+\Delta(1-\gamma)^{-5}\epsilon^{-4})$ trajectories.
Using $\Delta\leq 2R_{\max}/(1-\gamma)$ gives
$O((1-\gamma)^{-6}\epsilon^{-4})$ trajectories.

\paragraph{Variance-reduced policy gradient.}
Existing analyses of variance-reduced policy-gradient methods generally require uniformly bounded scores together with sufficient regularity of the scores as functions of the policy parameters. Uniform score Lipschitzness, often established through a bound on the Hessian of the log policy, is a common sufficient condition. To realize the oracle conditions required by \algo, we impose the following boundedness and mean-square continuity assumption on the policy scores.

\begin{assumption}[Bounded and mean-square-continuous scores]
\label{asm:score-regularity}
The policy class admits finite constants $C_1,L_S>0$ such that, for every
$\vtheta,\vtheta'\in\Theta$:
\begin{enumerate}
    \item[(i)] the scores are uniformly bounded,
    \[
    \sup_{s\in\mathcal S,\,a\in\mathcal A}
    \left\|
        \nabla\log\pi_{\vtheta}(a\mid s)
    \right\|
    \leq C_1;
    \]
    
    \item[(ii)] the scores are mean-square continuous,
    \[
    \sup_{s\in\mathcal S}
    \mathbb E_{a\sim\pi_{\vtheta}(\cdot\mid s)}
    \left[
        \left\|
            \nabla\log\pi_{\vtheta}(a\mid s)
            -
            \nabla\log\pi_{\vtheta'}(a\mid s)
        \right\|^2
    \right]
    \leq
    L_S^2\left\|\vtheta-\vtheta'\right\|^2.
    \]
\end{enumerate}
\end{assumption}

Assumption~\ref{asm:score-regularity} implies Assumption~\ref{asm:score-moment} with $c_1=C_1$. Hence, under bounded rewards, it also guarantees the bounded-variance one-policy oracle required for the large-batch updates of DEF-PG.

\begin{remark}
Assumption~\ref{asm:score-regularity}(ii) is weaker than requiring the scores to be uniformly Lipschitz for every state-action pair. In particular, a uniform bound on the spectral norm,
$\|\nabla^2\log\pi_{\vtheta}(a\mid s)\|_{2}$,
is sufficient but is not assumed. Assumption~\ref{asm:score-regularity}(i) is satisfied, for example, by softmax policies with suitably bounded logit gradients. It generally fails for Gaussian policies.
\end{remark}

We now verify that defensive importance sampling realizes the coupled oracle of Assumption~\ref{asm:mss-oracle}. For $k\geq1$, let
$
M_k
\coloneqq
\mathbb E[H_\gamma^k]$,
where $H_\gamma$ is the geometrically distributed trajectory length introduced in Remark~\ref{rem:geometric-stopping}. In particular, $M_k=O((1-\gamma)^{-k})$.
For $\alpha\in(0,1)$, also let $m_\alpha\coloneqq\min\{\alpha,1-\alpha\}$.

\begin{restatable}[Defensive coupling]{lemma}{vardef}
\label{lem:vardef}
Let $\vtheta_1,\vtheta_2\in\Theta$ and
\[
q_\alpha(\tau)
=
\alpha p(\tau\mid\vtheta_1)
+
(1-\alpha)p(\tau\mid\vtheta_2).
\]
For $\tau\sim q_\alpha$, define
\[
\omega_i(\tau)
=
\frac{p(\tau\mid\vtheta_i)}{q_\alpha(\tau)},
\qquad
\widehat g_i(\tau)
=
\omega_i(\tau)g(\vtheta_i,\tau),
\qquad i\in\{1,2\}.
\]
Under Assumption~\ref{asm:score-regularity}, the defensive weights satisfy
\begin{align*}
\mathbb V_{\tau\sim q_\alpha}[\omega_1(\tau)]
&\leq
\frac{(1-\alpha)^2C_1^2M_1}{2m_\alpha}
\left\|\vtheta_1-\vtheta_2\right\|^2,
\\
\mathbb V_{\tau\sim q_\alpha}[\omega_2(\tau)]
&\leq
\frac{\alpha^2C_1^2M_1}{2m_\alpha}
\left\|\vtheta_1-\vtheta_2\right\|^2.
\end{align*}
Moreover, under Assumptions~\ref{asm:bounded-rewards} and~\ref{asm:score-regularity},
\[
\mathbb E_{\tau\sim q_\alpha}
\left[
    \left\|
        \widehat g_1(\tau)-\widehat g_2(\tau)
    \right\|^2
\right]
\leq
C_\alpha
\left\|\vtheta_1-\vtheta_2\right\|^2,
\]
where
\[
C_\alpha
\coloneqq
\frac{R_{\max}^2}{m_\alpha}
\left(
    C_1^4M_5+2L_S^2M_4
\right).
\]
Finally,
\[
\mathbb E_{\tau\sim q_\alpha}[\widehat g_i(\tau)]
=
\nabla J(\vtheta_i),
\qquad i\in\{1,2\}.
\]
Consequently, the defensive pair $(\widehat g_1,\widehat g_2)$ realizes the coupled oracle of Assumption~\ref{asm:mss-oracle} with
$L_{\mathrm{pair}}^2=C_\alpha$.
\end{restatable}

Lemma~\ref{lem:vardef} establishes the oracle property needed by the variance-reduced update. Its proof uses neither Assumption~\ref{asm:var_iw} nor bounds on second derivatives of the policy scores.
Armed with this powerful technical tool, we can prove the following.

\begin{restatable}[Sample complexity of \textsc{Coupled-PAGE}]{lemma}{thmupper}
\label{thm:upper}
Suppose that Assumptions~\ref{asm:smooth-performance},
\ref{asm:bv-oracle}, and~\ref{asm:mss-oracle} hold, and let $\Delta
\coloneqq
J^\star -J(\vtheta_0)
<+\infty$.
For a target accuracy $\epsilon>0$, choose
\[
N
\geq
\max\left\{
    2,
    \left\lceil
        \frac{3\nu^2}{\epsilon^2}
    \right\rceil
\right\},
\qquad
1\leq B\leq N,
\qquad
p=\frac{1}{N},
\]
and
\[
\eta
=
\frac{1}{
    2\left(
        L+
        L_{\mathrm{pair}}
        \sqrt{\frac{N-1}{B}}
    \right)
}.
\]
If
\[
T
\geq
\max\left\{
    N,
    \left\lceil
        \frac{6\Delta}{\eta\epsilon^2}
    \right\rceil
\right\},
\]
then Algorithm~\ref{algo:coupled-page} returns a parameter satisfying $
\mathbb E
\left[
    \left\|
        \nabla J(\vtheta_{\mathrm{OUT}})
    \right\|
\right]
\leq\epsilon
$.
Its expected total number of oracle calls is at most
\[
N+
T\left(
    pN+(1-p)B
\right)
=
N+
T\left(
    1+B-\frac{B}{N}
\right).
\]
In particular, setting $B=1$ gives
\[
N_{\mathrm{TOT}}
=
O\left(
    \frac{\nu^2}{\epsilon^2}
    +
    \frac{\Delta L}{\epsilon^2}
    +
    \frac{\Delta L_{\mathrm{pair}}}{\epsilon^2}
    +
    \frac{\Delta L_{\mathrm{pair}}\nu}{\epsilon^3}
\right)
=
O(\epsilon^{-3}).
\]
\end{restatable}

\begin{theorem}[Sample complexity of DEF-PG]
\label{cor:def-pg-upper}
Suppose that Assumptions~\ref{asm:smooth-performance},
\ref{asm:bounded-rewards}, and~\ref{asm:score-regularity} hold. Then DEF-PG (Algorithm~\ref{algo_ppgdef}) satisfies the conclusions of Lemma~\ref{thm:upper} with
\[
\nu^2
=
\frac{8C_1^2R_{\max}^2}{(1-\gamma)^3},
\]
and
\[
L_{\mathrm{pair}}^2
=
C_\alpha
=
\frac{R_{\max}^2}{\min\{\alpha,1-\alpha\}}
\left(
    C_1^4M_5+2L_S^2M_4
\right).
\]
Consequently, DEF-PG uses $O(\epsilon^{-3})$ stopped trajectories to find an $\epsilon$-FOSP.
\end{theorem}

\begin{proof}
Assumption~\ref{asm:score-regularity} implies Assumption~\ref{asm:score-moment} with $c_1=C_1$. Proposition~\ref{prop:gpomdp_var} therefore verifies Assumption~\ref{asm:bv-oracle}, and Lemma~\ref{lem:vardef} verifies Assumption~\ref{asm:mss-oracle} for the defensive coupled samples. Applying Lemma~\ref{thm:upper} proves the result.
\end{proof}

The above result also provides a natural choice for the defensive parameter.

\begin{restatable}[Choice of the defensive parameter]{corollary}{corupper}
\label{cor:upper}
Among the defensive realizations in Theorem~\ref{cor:def-pg-upper}, the stated upper bound is minimized by $\alpha=1/2$.
\end{restatable}
\begin{remark}
    The choice $\alpha=1/2$ agrees with intuition, by the symmetry of the gradient estimator $v_{t+1}$ in Algorithm~\ref{algo_ppgdef}: any smaller choice of the defensive parameter would inject extra variance in the first term; any larger choice, in the third.
\end{remark}

Theorem~\ref{cor:def-pg-upper} matches the best known
$O(\epsilon^{-3})$ first-order rate~\citep[e.g.,][]{gargiani2022page}, without Assumption~\ref{asm:var_iw} or access to second-order information. 
Section~\ref{sec:lower} establishes, in a generalized black-box
policy-optimization model, lower bounds with matching dependence on
$\epsilon$: bounded-variance one-policy feedback yields the
$\epsilon^{-4}$ regime, whereas mean-square-smooth coupled two-policy
feedback yields the $\epsilon^{-3}$ regime.
Table~\ref{tab:summary} summarizes these results and separately identifies their concrete policy-gradient realizations.

\begin{table}[t]
\centering
\small
\caption{First-order sample complexity in the two oracle regimes. The upper bounds are Proposition~\ref{prop:one-policy-upper} and
Lemma~\ref{thm:upper}; the lower bounds are
Theorems~\ref{th:lower1} and~\ref{th:lower2}. Concrete assumptions are only employed in upper bounds. Assumption \ref{asm:var_iw} is never needed.}
\label{tab:summary}
\begin{tabularx}{\textwidth}{
    @{}l
    c
    l
    >{\centering\arraybackslash}X
    c
    c@{}}
\toprule
\bfseries Feedback
&
\shortstack{\bfseries Oracle\\\bfseries assumptions}
&
\bfseries Upper
&
\bfseries Lower
&
\bfseries \shortstack{\bfseries Concrete\\\bfseries Algorithm}
&
\shortstack{\bfseries Concrete\\\bfseries assumptions}
\\
\midrule
One-policy
&
\ref{asm:smooth-performance}, \ref{asm:bv-oracle}
&
$O(\epsilon^{-4})$
&
$\Omega(\epsilon^{-4})$
&
GPOMDP
&
\ref{asm:bounded-rewards}, \ref{asm:score-moment}
\\
Coupled
&
\ref{asm:smooth-performance}, \ref{asm:bv-oracle}, \ref{asm:mss-oracle}
&
$O(\epsilon^{-3})$
&
$\Omega(\epsilon^{-3})$
&
DEF-PG (Ours)
&
\ref{asm:bounded-rewards}, \ref{asm:score-regularity}
\\
\bottomrule
\end{tabularx}
\end{table}

%% file: src/Lower_Bounds.tex
\section{Sample Complexity Lower Bounds}\label{sec:lower}

In this section, we prove lower bounds on the sample complexity of first-order policy optimization in a black-box model where the learner selects policy parameters and receives only reward and policy-score feedback. We distinguish between bounded-variance one-policy feedback satisfying Assumption~\ref{asm:bv-oracle} and mean-square-smooth coupled two-policy feedback satisfying Assumptions~\ref{asm:bv-oracle} and~\ref{asm:mss-oracle}, obtaining lower bounds of $\Omega(\epsilon^{-4})$ and $\Omega(\epsilon^{-3})$, respectively. These lower bounds match the dependence on $\epsilon$ of the
corresponding upper bounds under the same oracle-level assumptions.
This matching is specific to the black-box feedback model and does not
by itself establish lower bounds for the classical observed-trajectory
MDP protocol. Nevertheless, the model captures the one-policy and
coupled first-order feedback structures underlying REINFORCE and
\algo, so the matching rates provide evidence that their different
complexities reflect a genuine distinction between these two forms of
policy-gradient feedback.
Our lower bounds build on the stochastic nonconvex optimization results of~\citet{arjevani2023lower}. All proofs are deferred to Appendix~\ref{app:proofs_lower}.

\paragraph{Feedback models. }
We first formalize the interaction model in which our lower bounds are
stated and then distinguish its one-policy and coupled feedback regimes.
\begin{definition}[Black-box policy-optimization model]
\label{def:black-box-policy-model}
Fix $\gamma\in(0,1)$ and a known measurable first-order decoder $\mathsf D$. A black-box policy-optimization instance specifies a parameter space $\Theta$, a differentiable stochastic policy class $\Pi_\Theta=\{\pi_{\vtheta}:\vtheta\in\Theta\}$, and a possibly parameter-dependent reward mechanism.
 Its performance function
is
\[
J(\vtheta)
\coloneqq
\mathbb E_{\vtheta}\!\left[
    \sum_{h=0}^{\infty}\gamma^hR_{h+1}
\right].
\]
We assume that $J$ is well defined and differentiable and that
$\mathbb E_{\vtheta}[\sum_{h\geq0}\gamma^h|R_{h+1}|]<+\infty$ for every
$\vtheta\in\Theta$.

When queried at $\vtheta$, the model internally generates hidden states
and actions satisfying $A_h\sim\pi_{\vtheta}(\cdot\mid S_h)$ and returns
only the reward-score record
\[
\mathcal Y(\vtheta;\zeta)
=
\left(R_{h+1},S_{\vtheta,h}\right)_{h\geq0},
\qquad
S_{\vtheta,h}
\coloneqq
\nabla\log\pi_{\vtheta}(A_h\mid S_h).
\]
The states, actions, policy densities, and internal random variables are
not returned.
The designated initial parameter is $\vtheta_0=\bm 0$. The completed first-order sample associated with a record is
$
\widehat g(\vtheta;\zeta)
\coloneqq
\mathsf D(\mathcal Y(\vtheta;\zeta))$.\footnote{
The decoder only identifies the stochastic-gradient statistic used to formulate the oracle conditions. The learner observes the complete record $\mathcal Y$ and is not restricted to using this statistic.}

\end{definition}

The policy class in Definition~\ref{def:black-box-policy-model} is genuine: the returned scores are the likelihood-ratio scores of the hidden actions actually sampled by the policy. However, the policy is available to the learner only through the black-box experiment. In particular, the learner cannot inspect the sampled actions, evaluate policy densities, or compute additional likelihood ratios.

\begin{definition}[One-policy feedback]
\label{def:one-policy-feedback}
At round $t$, the learner selects a single parameter $\vtheta_t$, possibly as a randomized function of all previous feedback, and receives an independent reward-score record
$\mathcal Y(\vtheta_t;\zeta_t)$. The instance belongs to the bounded-variance one-policy regime if the corresponding completed first-order samples satisfy Assumption~\ref{asm:bv-oracle}.
\end{definition}

The one-policy model captures methods that obtain feedback from a single
policy in each oracle call. Variance-reduced recursions additionally
compare feedback at two policies using shared randomness, which
motivates the following stronger interaction model.

\begin{definition}[Coupled two-policy feedback]
\label{def:coupled-feedback}
In addition to one-policy queries, the learner may select a pair
$(\vtheta_t,\vtheta_t')$. The model then uses common internal randomness $\zeta_t$ to return two reward-score records
$\mathcal Y_1(\vtheta_t,\vtheta_t';\zeta_t)$ and
$\mathcal Y_2(\vtheta_t,\vtheta_t';\zeta_t)$.

The marginal distribution of the first record is the same as that of a one-policy query at $\vtheta_t$, and the marginal distribution of the second is the same as that of a one-policy query at $\vtheta_t'$. Their associated completed samples are denoted by
$\widehat g_1(\vtheta_t,\vtheta_t';\zeta_t)$ and
$\widehat g_2(\vtheta_t,\vtheta_t';\zeta_t)$. The instance belongs to the coupled regime if these samples satisfy Assumption~\ref{asm:mss-oracle}. Different oracle calls use independent internal randomness.
\end{definition}

One one-policy experiment counts as one one-policy oracle call, and one jointly generated pair counts as one coupled oracle call.

\begin{remark}
\label{rem:black-box-scope}
The black-box model differs from the standard observed-trajectory
protocol because states and actions are hidden and the reward mechanism
may depend explicitly on the queried parameter. Nevertheless, its
performance function is the expected discounted return generated by
that reward mechanism.
\end{remark}

\paragraph{Gaussian policy realization.}
Let $Z\sim\mathsf P$ be independent of $\vtheta$, let
$f:\Theta\times\mathcal Z\to\Reals$ be differentiable in its first
argument, and define $F(\vtheta)\coloneqq\mathbb E[f(\vtheta,Z)]$.
We assume that $F$ is differentiable and that, for every
$\vtheta\in\Theta$,
$\mathbb E_Z[\nabla f(\vtheta,Z)]=\nabla F(\vtheta)$.

The construction uses an active state with policy
$\mathcal N(f(\vtheta,Z),1)$ and reward
$A_0-f(\vtheta,Z)$, followed by a payment state with a
parameter-independent dummy policy and reward $F(\vtheta)/\gamma$, and
then by a zero-reward absorbing state.
Its expected discounted return is therefore $F(\vtheta)$.

\begin{proposition}[Gaussian reward-score realization]
\label{prop:gaussian-realization}
There exists an infinite-horizon state-dependent Gaussian policy class
and a possibly parameter-dependent reward mechanism whose performance
function is $J(\vtheta)=F(\vtheta)$. Its observable reward-score record
is
\[
\mathcal Y(\vtheta;Z,\xi)
=
\left(
    (\xi,\xi\nabla f(\vtheta,Z)),
    (F(\vtheta)/\gamma,0),
    (0,0),(0,0),\ldots
\right),
\qquad
\xi\sim\mathcal N(0,1).
\]
The decoder
\[
\mathsf D(\mathcal Y)
=
\begin{cases}
S_{\vtheta,0}/R_1 &\text{if }
    R_1\neq0,\\
0
    &\text{otherwise}
\end{cases},
\]
satisfies $\widehat g(\vtheta;Z,\xi)=\nabla f(\vtheta,Z)$ almost surely.
\end{proposition}

Proposition~\ref{prop:gaussian-realization} embeds the
stochastic-gradient oracle associated with $f$ into the black-box
policy-optimization model. We next show that this embedding also
preserves the shared-seed coupling needed for variance reduction.

\begin{proposition}[Coupled Gaussian reward-score realization]
\label{prop:coupled-gaussian-realization}
Suppose that the one-policy samples satisfy
Assumption~\ref{asm:bv-oracle},
and that, for every $\vtheta,\vtheta'\in\Theta$,
\[
\mathbb E_Z\!\left[
    \|\nabla f(\vtheta,Z)-\nabla f(\vtheta',Z)\|^2
\right]
\leq
\ell^2\|\vtheta-\vtheta'\|^2.
\]
Running two copies of Proposition~\ref{prop:gaussian-realization} with
the same $Z$ and $\xi$ produces decoded samples satisfying
Assumption~\ref{asm:mss-oracle} with
$L_{\mathrm{pair}}=\ell$.
\end{proposition}

The one-policy and coupled oracle properties of the construction
follow directly from the corresponding moment conditions on
$\nabla f(\vtheta,Z)$. The following corollary shows the resulting
oracle parameters.

\begin{corollary}[Oracle parameters of the Gaussian realization]
\label{cor:gaussian-oracle-parameters}
If, for every $\vtheta\in\Theta$,
\[
\mathbb E_Z\!\left[
    \|\nabla f(\vtheta,Z)-\nabla F(\vtheta)\|^2
\right]
\leq\sigma^2,
\]
then the one-policy samples satisfy
Assumption~\ref{asm:bv-oracle} with $\nu^2=\sigma^2$.
Under the condition of
Proposition~\ref{prop:coupled-gaussian-realization}, the coupled samples
satisfy Assumption~\ref{asm:mss-oracle} with
$L_{\mathrm{pair}}^2=\ell^2$.
\end{corollary}

The proofs and the complete state, transition, policy, and reward
specification are given in Appendix~\ref{app:gaussian-realization}. In the following, we write $J^\star\coloneqq\sup_{\vtheta\in\Theta}J(\vtheta)$ and
$\Delta\coloneqq J^\star-J(\vtheta_0)$.

\paragraph{Sample complexity. }
We first consider the bounded-variance one-policy regime.

\begin{restatable}[One-policy oracle lower bound]{theorem}{badlower}
\label{th:lower1}
There exist numerical constants
$L_0,\nu_0,\Delta_0,\epsilon_0,c_{\mathrm{BV}}>0$,
independent of $\epsilon$, such that, for every
$\epsilon\in(0,\epsilon_0)$, there exist a dimension $d$ and a
distribution $\mathsf Q_{\epsilon}^{\mathrm{BV}}$ over black-box
policy-optimization instances on $\Theta=\Reals^d$, each having the
following properties:
\begin{enumerate}
    \item The policy class is the Gaussian policy
class described in
Proposition~\ref{prop:gaussian-realization};
    \item The performance function satisfies
    Assumption~\ref{asm:smooth-performance} with $L=L_0$ and has initial
    gap $\Delta\leq\Delta_0$;
    \item The completed one-policy samples satisfy
    Assumption~\ref{asm:bv-oracle} with $\nu=\nu_0$.
\end{enumerate}
For every possibly randomized adaptive algorithm operating under the
one-policy feedback model of
Definition~\ref{def:one-policy-feedback}, if
$T<c_{\mathrm{BV}}\epsilon^{-4}$, then
\[
\mathbb E_{\mathcal I\sim\mathsf Q_{\epsilon}^{\mathrm{BV}}}
\left[
    \left\|
        \nabla J_{\mathcal I}(\vtheta_{\mathrm{OUT}})
    \right\|
\right]
>
\epsilon,
\]
where the expectation also includes the black-box feedback and the
internal randomness of the algorithm. Consequently, for every such
algorithm there exists an instance in the support of
$\mathsf Q_{\epsilon}^{\mathrm{BV}}$ on which finding an
$\epsilon$-FOSP requires $\Omega(\epsilon^{-4})$ one-policy calls.
\end{restatable}

The constants in Theorem~\ref{th:lower1} do not depend on $\epsilon$; only the dimension of the hard instance may grow as $\epsilon$ decreases. The lower bound therefore matches the one-policy upper bound in Proposition~\ref{prop:one-policy-upper} under the same oracle-level assumptions. Under additional concrete assumptions, GPOMDP and REINFORCE realize this bounded-variance oracle, although those primitive assumptions are not imposed on the lower-bound construction.

We next consider the stronger feedback model of Definition~\ref{def:coupled-feedback}, in which two policy experiments can share the same internal randomness. This coupling is the policy-optimization analogue of the shared-seed two-point stochastic oracle studied by~\citet{arjevani2023lower}. In the coupled experiment of Proposition~\ref{prop:coupled-gaussian-realization}, the two hidden policy experiments share both $Z$ and $\xi$.

\begin{restatable}[Coupled two-policy oracle lower bound]{theorem}{goodlower}
\label{th:lower2}
There exist numerical constants
$L_0,\nu_0,L_{\mathrm{pair},0},\Delta_0,\epsilon_0,
c_{\mathrm{MSS}}>0$,
independent of $\epsilon$, such that, for every
$\epsilon\in(0,\epsilon_0)$, there exist a dimension $d$ and a
distribution $\mathsf Q_{\epsilon}^{\mathrm{MSS}}$ over black-box
policy-optimization instances on $\Theta=\Reals^d$, each having the
following properties:
\begin{enumerate}
    \item The policy class is the Gaussian policy
class described in
Propositions~\ref{prop:gaussian-realization}
and~\ref{prop:coupled-gaussian-realization};
    \item The performance function satisfies
    Assumption~\ref{asm:smooth-performance} with $L=L_0$ and has initial
    gap $\Delta\leq\Delta_0$;
    \item The completed one-policy samples satisfy
    Assumption~\ref{asm:bv-oracle} with $\nu=\nu_0$;
    \item The coupled samples satisfy Assumption~\ref{asm:mss-oracle}
    with $L_{\mathrm{pair}}=L_{\mathrm{pair},0}$.
\end{enumerate}
For every possibly randomized adaptive algorithm operating under the
coupled feedback model of Definition~\ref{def:coupled-feedback}, if the
total number of black-box calls satisfies
$T<c_{\mathrm{MSS}}\epsilon^{-3}$, then
\[
\mathbb E_{\mathcal I\sim\mathsf Q_{\epsilon}^{\mathrm{MSS}}}
\left[
    \left\|
        \nabla J_{\mathcal I}(\vtheta_{\mathrm{OUT}})
    \right\|
\right]
>
\epsilon,
\]
where the expectation also includes the black-box feedback and the
internal randomness of the algorithm. Here, either a one-policy
experiment or a coupled two-policy experiment counts as one black-box
call. Consequently, for every such algorithm there exists an instance
in the support of $\mathsf Q_{\epsilon}^{\mathrm{MSS}}$ on which finding
an $\epsilon$-FOSP requires $\Omega(\epsilon^{-3})$ black-box calls.
\end{restatable}

Together, Theorems~\ref{th:lower1} and~\ref{th:lower2} show that,
within the black-box model of
Definitions~\ref{def:black-box-policy-model}--\ref{def:coupled-feedback},
the optimal rate is
$\Theta(\epsilon^{-4})$ in the bounded-variance one-policy regime and
$\Theta(\epsilon^{-3})$ in the mean-square-smooth coupled regime. These
rates match the dependence on $\epsilon$ of the upper bounds in
Proposition~\ref{prop:one-policy-upper} and Lemma~\ref{thm:upper}.
Moreover, under standard policy-regularity conditions, REINFORCE and
\algo realize the respective oracle conditions and attain these upper
bounds.

This correspondence is at the oracle level and does not establish a
minimax separation under the classical MDP interaction model, which
reveals states and actions and permits the direct computation of
importance weights. Nevertheless, the black-box model captures the
one-policy and coupled first-order feedback structures underlying
REINFORCE and \algo. The matching rates therefore support the view that
the faster rate of \algo reflects a genuine benefit of off-policy feedback.

%% file: src/Related_Works.tex
\section{Related Works}\label{sec:disc}
In this section, we discuss and compare the most closely related results from the reinforcement learning literature.

\paragraph{Local convergence of first-order algorithms.}
The $O(\epsilon^{-4})$ convergence of REINFORCE/GPOMDP to a stationary point was known for a long time as it follows from standard stochastic approximation results~\citep{bertsekas1996neuro}, but a detailed and refined analysis can be found in the recent work by~\citet{yuan2022general}. For instance,~\citet{sutton1999policy} studied \emph{asymptotic} convergence under stronger regularity assumptions on the policy class.
The first applications of stochastic variance reduction to RL were for policy evaluation~\citep{du2017stochastic} and trust-region methods~\citep{xu2017stochastic}.
\citet{papini2018stochastic} were the first to apply stochastic variance reduction to a REINFORCE-like algorithm with the aim of improving the sample complexity for finding an $\epsilon$-FOSP, although they did not establish a better rate than $O(\epsilon^{-4})$ for their SVRPG algorithm.
SVRPG was based on SVRG~\citep{johnson2013accelerating}, and its analysis on previous works on nonconvex finite-sum optimization~\citep{reddi2016stochastic}, while SCSG~\citep{lei2017less}, sharing with SVRPG the \emph{infinite} sum nature of the objective function, was developed independently for supervised learning.
To analyze SVRPG, \citet{papini2018stochastic} introduced a uniform bound on the variance of ordinary importance weights, together with boundedness and smoothness conditions on the policy scores. These conditions later became standard in work adapting stochastic variance reduction to policy-gradient methods. They also assumed bounded variance of the on-policy gradient estimator; under bounded rewards and a bounded score second moment, this property follows from Proposition~\ref{prop:gpomdp_var}.
\citet{xu2019improved} later showed the sample complexity of SVRPG to be $O(\epsilon^{-10/3})$. The same authors~\citep{xu2020sample} proposed a new first-order algorithm, SRVR-PG, based on SARAH~\citep{nguyen2017sarah} and SPIDER~\citep{fang2018spider}, with $O(\epsilon^{-3})$ sample complexity, also considering constrained parameter spaces and parameter-based exploration. The $O(\epsilon^{-3})$ result was matched several times: by~\citet{yuan2020stochastic} with STORM-PG, based on STORM~\citep{cutkosky2019momentum}, which has the additional advantage of being \emph{single loop} (so is our algorithm); by~\citet{pham2020hybrid} with ProxHSPGA, also considering constrained and regularized problems; by~\citet{huang2020momentum} with IS-MBPG, based on STORM, considering adaptive learning rates; by~\citet{huang2022bregman}, that in their  VR-BGPO algorithm applied stochastic variance reduction to a mirror-descent scheme; and by~\citet{gargiani2022page} with PAGE-PG, based on PAGE, a more recent algorithm for nonconvex optimization~\citep{li2021page}. The last four algorithms have the advantage of using an $O(1)$ mini-batch size for the stochastic-gradient updates, compared to $O(\epsilon^{-1/2})$ for SRVR-PG (so does our algorithm). All of these works establish $O(\epsilon^{-3})$ sample complexity under Assumption~\ref{asm:var_iw}. To see why this assumption is unrealistic, consider the detailed discussion on the variance of importance weights for policy optimization by~\citet{metelli2020importance}. This issue was addressed by~\citet{zhang2021convergence}. Their TSIVR-PG algorithm, based on SARAH/SPIDER, achieves $O(\epsilon^{-3})$ sample complexity without any assumption on the variance of importance weights, by truncating the gradients. They also consider general utility functions that cannot be written as a sum of rewards. However, their analysis is tailored to softmax policies, which are not suitable for continuous control. Like TSIVR-PG, DEF-PG achieves $O(\epsilon^{-3})$ sample complexity without any assumption on the variance of ordinary importance weights. Its concrete guarantee applies to policy classes satisfying the boundedness and mean-square-continuity conditions in Assumption~\ref{asm:score-regularity}, without gradient truncation. The version presented in this paper is most closely related to PAGE-PG.
The same defensive-mixture construction can also be inserted into the
importance-weighted recursions of SVRPG, SRVR-PG, and STORM-PG.
Appendix~\ref{app:variants} gives the resulting pseudocode.

\paragraph{Local convergence of second order algorithms.}
Another way to get rid of Assumption~\ref{asm:var_iw} is to use second-order information, which avoids the use of importance weights entirely.\footnote{Curiously,~\citet{yang2022policy} claim $O(\epsilon^{-3})$ sample complexity for their VRMPO algorithm without using importance weights nor second-order information. We suspect that their proof may be affected by a subtle oversight where the distribution shift is ignored~(see Equation 51 in the proof of their Theorem 4.1).} The first to do so were~\citet{shen2019hessian}. Their HAPG algorithm enjoys $O(\epsilon^{-3})$ sample complexity for $\epsilon$-FOSP convergence and was developed independently from SRVR-PG. Efficient implementations allow to keep the per-iteration computational complexity linear in the dimension $d$ and the horizon $T$. A similar solution was proposed by~\citet{salehkaleybar2022adaptive}. Their SHARP algorithm has the advantage of using $O(1)$ batch sizes for the stochastic gradient updates, compared to $O(\epsilon^{-1})$ for HAPG. 

\paragraph{Lower bounds for local convergence.}
The $O(\epsilon^{-3})$ rate of SRVR-PG and its successors is believed to be optimal because of the matching lower bound for first-order nonconvex stochastic optimization established by~\citet{arjevani2023lower}. That result does not directly imply a policy-gradient lower bound: the feedback and sampling protocols differ, and a standard reinforcement-learning objective is generated by a policy interacting with a parameter-independent environment.

Our lower bounds provide an exact reduction in the black-box model of
Definition~\ref{def:black-box-policy-model}. The construction uses
genuine likelihood-ratio scores from the state-dependent Gaussian
policy class of Proposition~\ref{prop:gaussian-realization}, and its
objective is the expected discounted infinite-horizon return. The
reward mechanism may, however, depend explicitly on the queried policy
parameter. Establishing analogous lower bounds with parameter-independent
rewards under the standard observed-trajectory protocol remains open.

\paragraph{Global convergence.}
Several recent works study the convergence of policy gradient algorithms to a globally optimal policy, that is, $J^\star -J(\vtheta)\le \varepsilon$. Additional assumptions are needed since the performance function is, in general, nonconvex. A common such assumption is Fisher non-degeneracy, introduced by~\citet{liu2020improved}. Under this assumption, the latter established $O(\varepsilon^{-4})$ sample complexity for GPOMDP,  and $O(\varepsilon^{-3})$ for Natural Policy Gradient~\citep[NPG,][]{kakade2001natural}, SRVR-PG, and their own algorithm SRVR-NPG, that achieves a better dependence on variance upper bounds by combining NPG with SARAH-style variance reduction. Notably, Assumption~\ref{asm:var_iw} is still required by their analysis. Note that these sample complexity results are not comparable with the ones for $\epsilon$-FOSP, since $\epsilon$ and $\varepsilon$ measure two different things (gradient norms and performance gaps, respectively). Moreover, it is a counter-intuitive but well-known fact of optimization that fast convergence to $\varepsilon$-global-optima does not necessarily imply fast convergence to $\epsilon$-FOSP~\citep[e.g.,][]{zhu2018how}. Indeed,~\citet{fatkhullin2023improved} designed a first-order algorithm (N-PG-IGT) that, under Fisher non-degeneracy, converges to a globally optimal policy at a rate of $O(\varepsilon^{-2.5})$ \emph{without using importance sampling}. However, its rate to $\epsilon$-FOSP is $O(\epsilon^{-3.5})$, even worse than SVRPG. They also proposed a second-order algorithm (N-HARPG) with $O(\varepsilon^{-2})$ global convergence that matches the $O(\epsilon^{-3})$ FOSP convergence of HAPG and SHARP. Several other works consider global convergence, but a further discussion would be out of scope for this paper. Refer to~\citep{fatkhullin2023improved} for recent coverage of the topic. The classic $\Omega(\varepsilon^{-2})$ lower bound by~\citet{azar2013minimax} holds for this setting.

%% file: src/Conclusion.tex
\section{Conclusion}

We introduced DEF-PG, a first-order variance-reduced policy-gradient algorithm based on defensive importance sampling. At the oracle level, its recursion is an instance of \textsc{Coupled-PAGE} and achieves an $O(\epsilon^{-3})$ sample complexity under smoothness of the performance function, bounded-variance one-policy feedback, and mean-square-smooth coupled two-policy feedback. For the concrete discounted policy-gradient problem, we showed that bounded rewards and the boundedness and mean-square continuity of the policy scores are sufficient to realize these oracle conditions. In particular, DEF-PG does not require Assumption~\ref{asm:var_iw} on the variance of ordinary importance weights.

% Our numerical results demonstrate the practical advantages of defensive importance sampling relative to comparable algorithms based on ordinary importance weighting. In particular, they support the view that uncontrolled importance weights can substantially degrade empirical performance.

We also established lower bounds in the black-box policy-optimization
model. Within this model, the optimal dependence on the stationarity
tolerance is $\Theta(\epsilon^{-4})$ under bounded-variance one-policy
feedback and $\Theta(\epsilon^{-3})$ under mean-square-smooth coupled
two-policy feedback. Under standard policy-regularity conditions,
REINFORCE and \algo realize the respective oracle conditions and attain
the corresponding upper bounds. The matching exponents therefore
provide oracle-level evidence that the faster rate of variance-reduced
policy gradient is optimal and separated from that of REINFORCE.
These lower bounds do not, however, establish minimax optimality under
the primitive MDP assumptions used by the concrete policy-gradient
upper bounds. It remains open whether the same rates can be proved
necessary under the standard observed-trajectory protocol with
parameter-independent transition and reward mechanisms, where the
learner observes states and actions and may exploit information not
captured by our black-box formulation. It would also be interesting to
obtain lower-bound constructions satisfying the concrete reward and
score assumptions of Theorem~\ref{cor:def-pg-upper}, and to determine
the optimal dependence on the effective horizon $(1-\gamma)^{-1}$.

Further directions include studying more natural policy classes, extending the lower bounds to algorithms using second-order information following~\citet{arjevani2020second}, and analyzing the global-convergence properties of DEF-PG under conditions such as Fisher non-degeneracy.

%% file: src/Appendix.tex
\section{Proofs of Upper Bounds}\label{app:proofs_upper}

\subsection{Proofs of Propositions}\label{app:proofprop}
The results proved here are rather standard, but proofs are given for the sake of completeness.

\begin{proof}[Proof of Proposition~\ref{prop:one-policy-upper}]
By $L$-smoothness,
$\|\nabla J(\vtheta_0)\|^2\leq 2L\Delta$. Therefore, if
$2L\Delta\leq\epsilon^2$, the initial parameter is already an
$\epsilon$-FOSP.

Otherwise, let
$m=\max\{1,\lceil 2\nu^2/\epsilon^2\rceil\}$ and
$T=\lceil 4L\Delta/\epsilon^2\rceil$. At iteration $t$, average $m$
independent calls to $\mathcal O_{\mathrm{BV}}$ to obtain
$\overline g_t$, and set
$\vtheta_{t+1}=\vtheta_t+L^{-1}\overline g_t$. Writing
$e_t=\overline g_t-\nabla J(\vtheta_t)$, we have
$\mathbb E[e_t\mid\vtheta_t]=0$ and
$\mathbb E[\|e_t\|^2\mid\vtheta_t]\leq\nu^2/m$.
The smoothness inequality therefore gives
\[
\mathbb E[J(\vtheta_{t+1})\mid\vtheta_t]
\geq
J(\vtheta_t)
+
\frac{1}{2L}\|\nabla J(\vtheta_t)\|^2
-
\frac{\nu^2}{2Lm}.
\]
Summing over $t=0,\ldots,T-1$ and using
$J(\vtheta_T)-J(\vtheta_0)\leq\Delta$ yields
\[
\frac{1}{T}
\sum_{t=0}^{T-1}
\mathbb E[\|\nabla J(\vtheta_t)\|^2]
\leq
\frac{2L\Delta}{T}
+
\frac{\nu^2}{m}
\leq
\epsilon^2.
\]
Choosing $\vtheta_{\mathrm{OUT}}$ uniformly from
$\{\vtheta_t\}_{t=0}^{T-1}$ and applying Jensen's inequality gives the
claimed expected-norm guarantee.

Finally, since the nontrivial case satisfies
$\epsilon^2<2L\Delta$, the total number of oracle calls is
\[
Tm
=
O\left(
    \frac{\Delta L}{\epsilon^2}
    +
    \frac{\Delta L\nu^2}{\epsilon^4}
\right).
\]
\end{proof}

\begin{proof}[Proof of Proposition~\ref{prop:gpomdp_var}]
For $h\geq0$, define
\[
Z_h(\vtheta)
\coloneqq
\sum_{\ell=0}^{h}
\nabla\log\pi_{\vtheta}(A_\ell\mid S_\ell),
\qquad
I_h
\coloneqq
\mathbf 1\{H_\gamma>h\}.
\]
The stopped GPOMDP estimator can be written as
\[
g(\vtheta,\tau)
=
\sum_{h\geq0}I_hR_{h+1}Z_h(\vtheta).
\]

Since the stopping rule is independent of the trajectory and
$\mathbb E[I_h]=\gamma^h$,
\[
\mathbb E[g(\vtheta,\tau)]
=
\sum_{h\geq0}
\gamma^h
\mathbb E\left[
    R_{h+1}Z_h(\vtheta)
\right]
=
\nabla J(\vtheta).
\]

The score increments are martingale differences under the on-policy trajectory distribution. Therefore, Assumption~\ref{asm:score-moment} gives
\[
\mathbb E\left[
    \left\|Z_h(\vtheta)\right\|^2
\right]
\leq
(h+1)c_1^2.
\]
Using Minkowski's inequality in $L^2$, bounded rewards, and the independence of $I_h$,
\begin{align*}
\left(
    \mathbb E\left[
        \left\|g(\vtheta,\tau)\right\|^2
    \right]
\right)^{1/2}
&\leq
\sum_{h\geq0}
\left(
    \mathbb E\left[
        I_hR_{h+1}^2
        \left\|Z_h(\vtheta)\right\|^2
    \right]
\right)^{1/2}
\\
&\leq
R_{\max}c_1
\sum_{h\geq0}
\gamma^{h/2}\sqrt{h+1}.
\end{align*}
Letting $q=\sqrt{\gamma}$ and applying Cauchy--Schwarz,
\[
\left(
    \sum_{h\geq0}q^h\sqrt{h+1}
\right)^2
\leq
\left(\sum_{h\geq0}q^h\right)
\left(\sum_{h\geq0}(h+1)q^h\right)
=
\frac{1}{(1-q)^3}.
\]
Since
\[
1-\sqrt{\gamma}
=
\frac{1-\gamma}{1+\sqrt{\gamma}}
\geq
\frac{1-\gamma}{2},
\]
we obtain
\[
\mathbb E\left[
    \left\|g(\vtheta,\tau)\right\|^2
\right]
\leq
\frac{8c_1^2R_{\max}^2}{(1-\gamma)^3}.
\]
Finally,
\[
\mathbb E
\left[
    \left\|
        g(\vtheta,\tau)-\nabla J(\vtheta)
    \right\|^2
\right]
\leq
\mathbb E
\left[
    \left\|g(\vtheta,\tau)\right\|^2
\right],
\]
which proves the result.
\end{proof}

\subsection{Proof of Lemma~\ref{thm:upper}}

We first prove two auxiliary lemmas.

\begin{lemma}[Weighted Hellinger continuity]
\label{lem:weighted-hellinger}
Suppose that Assumption~\ref{asm:score-regularity}(i) holds. For every integer $r\geq0$ and every $\vtheta_1,\vtheta_2\in\Theta$,
\[
\int_{\mathcal T}
H_\gamma(\tau)^r
\left(
    \sqrt{p(\tau\mid\vtheta_1)}
    -
    \sqrt{p(\tau\mid\vtheta_2)}
\right)^2
\de\tau
\leq
\frac{C_1^2M_{r+1}}{4}
\left\|\vtheta_1-\vtheta_2\right\|^2.
\]
\end{lemma}

\begin{proof}
Let
\[
\vtheta_t
=
(1-t)\vtheta_2+t\vtheta_1,
\qquad
\boldsymbol{\delta}
=
\vtheta_1-\vtheta_2,
\qquad
t\in[0,1].
\]
Convexity of $\Theta$ guarantees that $\vtheta_t\in\Theta$. Because the stopping rule and the environment are independent of the policy parameters,
\[
\nabla\log p(\tau\mid\vtheta_t)
=
\sum_{h=0}^{H_\gamma(\tau)-1}
\nabla\log\pi_{\vtheta_t}(A_h\mid S_h).
\]
Consequently,
\[
\frac{\de}{\de t}
\sqrt{p(\tau\mid\vtheta_t)}
=
\frac{1}{2}
\sqrt{p(\tau\mid\vtheta_t)}
\left\langle
    \boldsymbol{\delta},
    \nabla\log p(\tau\mid\vtheta_t)
\right\rangle.
\]

By the fundamental theorem of calculus and Cauchy--Schwarz,
\begin{align*}
&\left(
    \sqrt{p(\tau\mid\vtheta_1)}
    -
    \sqrt{p(\tau\mid\vtheta_2)}
\right)^2
\\
&\qquad\leq
\frac{\|\boldsymbol{\delta}\|^2}{4}
\int_0^1
p(\tau\mid\vtheta_t)
\left\|
    \nabla\log p(\tau\mid\vtheta_t)
\right\|^2
\de t.
\end{align*}
Multiplying by $H_\gamma(\tau)^r$ and integrating over $\tau$ yields
\begin{align*}
&\int_{\mathcal T}
H_\gamma(\tau)^r
\left(
    \sqrt{p(\tau\mid\vtheta_1)}
    -
    \sqrt{p(\tau\mid\vtheta_2)}
\right)^2
\de\tau
\\
&\qquad\leq
\frac{\|\boldsymbol{\delta}\|^2}{4}
\int_0^1
\mathbb E_{\tau\sim p(\cdot\mid\vtheta_t)}
\left[
    H_\gamma^r
    \left\|
        \sum_{h=0}^{H_\gamma-1}
        \nabla\log\pi_{\vtheta_t}(A_h\mid S_h)
    \right\|^2
\right]
\de t.
\end{align*}
The score increments are martingale differences. Conditionally on $H_\gamma$, Assumption~\ref{asm:score-regularity}(i) therefore gives
\[
\mathbb E_{\tau\sim p(\cdot\mid\vtheta_t)}
\left[
    \left\|
        \sum_{h=0}^{H_\gamma-1}
        \nabla\log\pi_{\vtheta_t}(A_h\mid S_h)
    \right\|^2
    \,\middle|\,
    H_\gamma
\right]
\leq
C_1^2H_\gamma.
\]
Since the distribution of $H_\gamma$ is independent of $\vtheta_t$,
\[
\mathbb E
\left[
    H_\gamma^r
    \left\|
        \nabla\log p(\tau\mid\vtheta_t)
    \right\|^2
\right]
\leq
C_1^2M_{r+1}.
\]
Substitution proves the claim.
\end{proof}

\vardef*
\begin{proof}
For brevity, write
\[
p_i(\tau)=p(\tau\mid\vtheta_i),
\qquad
q_\alpha(\tau)=\alpha p_1(\tau)+(1-\alpha)p_2(\tau),
\qquad
m_\alpha=\min\{\alpha,1-\alpha\}.
\]
Since
\[
q_\alpha(\tau)
\geq
m_\alpha\bigl(p_1(\tau)+p_2(\tau)\bigr),
\]
we have
\begin{align}
\frac{(p_1-p_2)^2}{q_\alpha}
&\leq
\frac{1}{m_\alpha}
\frac{(p_1-p_2)^2}{p_1+p_2}
\nonumber\\
&=
\frac{1}{m_\alpha}
\left(\sqrt{p_1}-\sqrt{p_2}\right)^2
\frac{\left(\sqrt{p_1}+\sqrt{p_2}\right)^2}{p_1+p_2}
\nonumber\\
&\leq
\frac{2}{m_\alpha}
\left(\sqrt{p_1}-\sqrt{p_2}\right)^2.
\label{eq:triangular-hellinger}
\end{align}

Because $\mathbb E_{q_\alpha}[\omega_i]=1$ and
\[
p_1-q_\alpha=(1-\alpha)(p_1-p_2),
\qquad
p_2-q_\alpha=-\alpha(p_1-p_2),
\]
Equation~\eqref{eq:triangular-hellinger} and Lemma~\ref{lem:weighted-hellinger} with $r=0$ give
\begin{align*}
\mathbb V_{q_\alpha}[\omega_1]
&=
(1-\alpha)^2
\int_{\mathcal T}
\frac{(p_1-p_2)^2}{q_\alpha}
\de\tau
\\
&\leq
\frac{(1-\alpha)^2C_1^2M_1}{2m_\alpha}
\left\|\vtheta_1-\vtheta_2\right\|^2,
\\
\mathbb V_{q_\alpha}[\omega_2]
&\leq
\frac{\alpha^2C_1^2M_1}{2m_\alpha}
\left\|\vtheta_1-\vtheta_2\right\|^2.
\end{align*}

We now prove the coupled-estimator bound. Let $g_i(\tau)=g(\vtheta_i,\tau)$.
Bounded rewards and Assumption~\ref{asm:score-regularity}(i) imply the pathwise estimate
\begin{align}
\|g_i(\tau)\|
&\leq
R_{\max}C_1
\sum_{h=0}^{H_\gamma-1}(h+1)
\nonumber\\
&\leq
R_{\max}C_1H_\gamma^2.
\label{eq:pathwise-g-bound}
\end{align}

Moreover,
\[
g_1(\tau)-g_2(\tau)
=
\sum_{\ell=0}^{H_\gamma-1}
\left(
    \sum_{h=\ell}^{H_\gamma-1}R_{h+1}
\right)
\left(
    \nabla\log\pi_{\vtheta_1}(A_\ell\mid S_\ell)
    -
    \nabla\log\pi_{\vtheta_2}(A_\ell\mid S_\ell)
\right).
\]
Using Cauchy--Schwarz and Assumption~\ref{asm:score-regularity}(ii), for $i\in\{1,2\}$,
\begin{equation}
\mathbb E_{\tau\sim p_i}
\left[
    \|g_1(\tau)-g_2(\tau)\|^2
    \,\middle|\,
    H_\gamma
\right]
\leq
R_{\max}^2L_S^2H_\gamma^4
\left\|\vtheta_1-\vtheta_2\right\|^2.
\label{eq:g-difference-bound}
\end{equation}
The case $i=2$ follows by interchanging $\vtheta_1$ and $\vtheta_2$ in Assumption~\ref{asm:score-regularity}(ii).

Decompose the coupled difference as
\[
\widehat g_1-\widehat g_2
=
(\omega_1-\omega_2)g_1
+
\omega_2(g_1-g_2).
\]
For the first term, Equations~\eqref{eq:triangular-hellinger} and~\eqref{eq:pathwise-g-bound}, followed by Lemma~\ref{lem:weighted-hellinger} with $r=4$, give
\begin{align*}
\mathbb E_{q_\alpha}
\left[
    \|(\omega_1-\omega_2)g_1\|^2
\right]
&=
\int_{\mathcal T}
\frac{(p_1-p_2)^2}{q_\alpha}
\|g_1\|^2
\de\tau
\\
&\leq
\frac{2R_{\max}^2C_1^2}{m_\alpha}
\int_{\mathcal T}
H_\gamma^4
\left(\sqrt{p_1}-\sqrt{p_2}\right)^2
\de\tau
\\
&\leq
\frac{R_{\max}^2C_1^4M_5}{2m_\alpha}
\left\|\vtheta_1-\vtheta_2\right\|^2.
\end{align*}
For the second term, since
$q_\alpha\geq(1-\alpha)p_2$,
Equation~\eqref{eq:g-difference-bound} gives
\begin{align*}
\mathbb E_{q_\alpha}
\left[
    \|\omega_2(g_1-g_2)\|^2
\right]
&=
\int_{\mathcal T}
\frac{p_2^2}{q_\alpha}
\|g_1-g_2\|^2
\de\tau
\\
&\leq
\frac{1}{1-\alpha}
\mathbb E_{p_2}
\left[
    \|g_1-g_2\|^2
\right]
\\
&\leq
\frac{R_{\max}^2L_S^2M_4}{m_\alpha}
\left\|\vtheta_1-\vtheta_2\right\|^2.
\end{align*}
Therefore,
\begin{align*}
\mathbb E_{q_\alpha}
\left[
    \|\widehat g_1-\widehat g_2\|^2
\right]
&\leq
2\mathbb E_{q_\alpha}
\left[
    \|(\omega_1-\omega_2)g_1\|^2
\right]
+
2\mathbb E_{q_\alpha}
\left[
    \|\omega_2(g_1-g_2)\|^2
\right]
\\
&\leq
\frac{R_{\max}^2}{m_\alpha}
\left(
    C_1^4M_5+2L_S^2M_4
\right)
\left\|\vtheta_1-\vtheta_2\right\|^2.
\end{align*}

Finally, change of measure gives
\[
\mathbb E_{q_\alpha}[\widehat g_i]
=
\mathbb E_{p_i}[g_i]
=
\nabla J(\vtheta_i),
\qquad i\in\{1,2\}.
\]
This proves all the claims.
\end{proof}

The proof follows the standard PAGE error-recursion argument,
using Assumptions~\ref{asm:bv-oracle} and~\ref{asm:mss-oracle}
to control the refresh and recursive updates, respectively.

\thmupper*
\begin{proof}
Let
\[
e_t
\coloneqq
v_t-\nabla J(\vtheta_t).
\]
Conditionally on the history up to iteration $t$, a large-batch refresh consists of $N$ independent calls to the oracle of Assumption~\ref{asm:bv-oracle}. Hence, on a refresh iteration,
\[
\mathbb E
\left[
    \|e_{t+1}\|^2
    \,\middle|\,
    \vtheta_{t+1}
\right]
\leq
\frac{\nu^2}{N}.
\]

On a recursive iteration, the two outputs of each coupled call are unbiased for
$\nabla J(\vtheta_{t+1})$ and $\nabla J(\vtheta_t)$, respectively. Since the $B$ coupled calls are independent, Assumption~\ref{asm:mss-oracle} gives
\[
\mathbb E[\|e_{t+1}\|^2]
\leq
\mathbb E[\|e_t\|^2]
+
\frac{L_{\mathrm{pair}}^2}{B}
\mathbb E
\left[
    \|\vtheta_{t+1}-\vtheta_t\|^2
\right].
\]
Combining the two cases and using
$\vtheta_{t+1}-\vtheta_t=\eta v_t$ yields
\begin{equation}
\label{eq:oracle-error-recursion}
\mathbb E[\|e_{t+1}\|^2]
\leq
(1-p)\mathbb E[\|e_t\|^2]
+
\frac{(1-p)L_{\mathrm{pair}}^2\eta^2}{B}
\mathbb E[\|v_t\|^2]
+
\frac{p\nu^2}{N}.
\end{equation}

Since $v_0$ is the average of $N$ independent one-policy oracle calls,
\[
\mathbb E[\|e_0\|^2]
\leq
\frac{\nu^2}{N}.
\]
Summing Equation~\eqref{eq:oracle-error-recursion} for
$t=0,\ldots,T-1$ gives
\begin{equation}
\label{eq:summed-oracle-error}
\sum_{t=0}^{T-1}
\mathbb E[\|e_t\|^2]
\leq
\frac{\nu^2}{pN}
+
\frac{(1-p)L_{\mathrm{pair}}^2\eta^2}{pB}
\sum_{t=0}^{T-1}
\mathbb E[\|v_t\|^2]
+
\frac{T\nu^2}{N}.
\end{equation}

By Assumption~\ref{asm:smooth-performance},
\begin{align}
J(\vtheta_{t+1})
&\geq
J(\vtheta_t)
+
\eta
\left\langle
    \nabla J(\vtheta_t),
    v_t
\right\rangle
-
\frac{L\eta^2}{2}\|v_t\|^2
\nonumber\\
&=
J(\vtheta_t)
+
\frac{\eta}{2}
\|\nabla J(\vtheta_t)\|^2
-
\frac{\eta}{2}
\|e_t\|^2
+
\frac{\eta}{2}
(1-L\eta)\|v_t\|^2.
\label{eq:smooth-ascent}
\end{align}
Summing Equation~\eqref{eq:smooth-ascent}, using
$J(\vtheta_T)-J(\vtheta_0)\leq\Delta$, and substituting
Equation~\eqref{eq:summed-oracle-error}, we obtain
\begin{align*}
\frac{\eta}{2}
\sum_{t=0}^{T-1}
\mathbb E
\left[
    \|\nabla J(\vtheta_t)\|^2
\right]
&\leq
\Delta
+
\frac{\eta\nu^2}{2pN}
+
\frac{\eta T\nu^2}{2N}
\\
&\quad+
\frac{\eta}{2}
\left(
    \frac{(1-p)L_{\mathrm{pair}}^2\eta^2}{pB}
    -
    (1-L\eta)
\right)
\sum_{t=0}^{T-1}
\mathbb E[\|v_t\|^2].
\end{align*}

For $p=1/N$, let
\[
a
\coloneqq
L_{\mathrm{pair}}
\sqrt{\frac{N-1}{B}}.
\]
The selected step size satisfies
\[
L\eta+a^2\eta^2
\leq
\frac{1}{2},
\]
so the coefficient of
$\sum_t\mathbb E[\|v_t\|^2]$ is non-positive. Therefore,
\[
\frac{1}{T}
\sum_{t=0}^{T-1}
\mathbb E
\left[
    \|\nabla J(\vtheta_t)\|^2
\right]
\leq
\frac{2\Delta}{\eta T}
+
\frac{\nu^2}{pNT}
+
\frac{\nu^2}{N}.
\]
By the choices of $N$, $p$, and $T$, each term on the right-hand side is at most
$\epsilon^2/3$. Since $\vtheta_{\mathrm{OUT}}$ is selected uniformly,
\[
\mathbb E
\left[
    \|\nabla J(\vtheta_{\mathrm{OUT}})\|^2
\right]
\leq
\epsilon^2.
\]
Jensen's inequality gives the guarantee in expected norm.

Initialization costs $N$ one-policy calls. Every subsequent iteration uses, in expectation,
\[
pN+(1-p)B
=
1+B-\frac{B}{N}
\]
oracle calls, proving the stated complexity.
\end{proof}

\corupper*
\begin{proof}
The quantities $M_4$, $M_5$, $R_{\max}$, $C_1$, and $L_S$ are independent of $\alpha$. Hence, the dependence of the bound in Lemma~\ref{lem:vardef} on the defensive parameter is entirely through
\[
\frac{1}{m_\alpha}
=
\frac{1}{\min\{\alpha,1-\alpha\}}.
\]
This quantity is minimized when
\[
\alpha=1-\alpha,
\]
that is, at $\alpha=1/2$.
\end{proof}

\section{Proofs of the Lower Bounds}
\label{app:proofs_lower}

We first record the consequences of the stochastic-optimization lower bounds of~\citet{arjevani2023lower} needed here. We then show that their oracle feedback can be represented exactly by the reward-score experiments of Propositions~\ref{prop:gaussian-realization} and~\ref{prop:coupled-gaussian-realization}.

\subsection{Source stochastic-optimization lower bounds}
\label{app:source-lower-bounds}

The oracle model of~\citet[Section~2]{arjevani2023lower} returns the exact objective value together with a stochastic gradient. For a batch
$(\vtheta_1,\ldots,\vtheta_K)$, it draws one random seed $Z$ and returns
\[
\left(
    J(\vtheta_k),
    g(\vtheta_k,Z)
\right)_{k=1}^K.
\]
Thus, the $K$ stochastic-gradient evaluations within one round share the same seed. Their complexity measure counts the number of such rounds, with $K$ treated as fixed.

\begin{proposition}[Consequence of \citet{arjevani2023lower}]
\label{prop:arjevani-specialization}
There exist numerical constants
$\Delta_0,L_0,\sigma_0,\overline L_0,
\epsilon_0,\overline c_{\mathrm{BV}},
\overline c_{\mathrm{MSS}}>0$,
independent of $\epsilon$, with the following properties.
For every $\epsilon\in(0,\epsilon_0)$, there is a dimension $d$ and a
distribution $\mathsf P_{\epsilon}^{\mathrm{BV}}$ over
stochastic-optimization instances. Each instance consists of a
differentiable maximization objective
$J:\Reals^d\to\Reals$, a seed distribution $\mathsf P$ on
$\mathcal Z$, and a sample objective
$f:\Reals^d\times\mathcal Z\to\Reals$ such that
\[
J(\vtheta)
=
\mathbb E_{Z\sim\mathsf P}[f(\vtheta,Z)],
\qquad
g(\vtheta,Z)
=
\nabla f(\vtheta,Z).
\]
Every instance in the support of
$\mathsf P_{\epsilon}^{\mathrm{BV}}$ satisfies:
\begin{enumerate}
    \item $J^\star-J(\bm 0)\leq\Delta_0$ and $J$ is $L_0$-smooth;
    \item $\mathbb E_Z[g(\vtheta,Z)]=\nabla J(\vtheta)$;
    \item
    $\mathbb E_Z[
        \|g(\vtheta,Z)-\nabla J(\vtheta)\|^2
    ]\leq\sigma_0^2$.
\end{enumerate}
For every randomized adaptive algorithm using the oracle with $K=1$,
if the source instance is drawn from
$\mathsf P_{\epsilon}^{\mathrm{BV}}$ and the algorithm uses fewer than
$\overline c_{\mathrm{BV}}\epsilon^{-4}$ rounds, then
$\mathbb E[
    \|\nabla J(\vtheta_{\mathrm{OUT}})\|
]
>
\epsilon$. Here and below, the expectation includes the random draw of the source
instance, the oracle randomness, and the internal randomness of the
algorithm.

There is also a distribution
$\mathsf P_{\epsilon}^{\mathrm{MSS}}$ over
stochastic-optimization instances satisfying the preceding properties
and
\[
\mathbb E_Z\!\left[
    \left\|
        g(\vtheta,Z)-g(\vtheta',Z)
    \right\|^2
\right]
\leq
\overline L_0^2
\left\|\vtheta-\vtheta'\right\|^2.
\]
For every randomized adaptive algorithm using the shared-seed oracle
with $K=2$, if the source instance is drawn from
$\mathsf P_{\epsilon}^{\mathrm{MSS}}$ and the algorithm uses fewer than
$\overline c_{\mathrm{MSS}}\epsilon^{-3}$ rounds, then $\mathbb E[
    \|\nabla J(\vtheta_{\mathrm{OUT}})\|
]
>
\epsilon$.
\end{proposition}

\begin{proof}
Theorem~3 of~\citet{arjevani2023lower} lower-bounds the
distributional complexities defined in their Equations~(7)--(8).
In these definitions, the suprema over the oracle and the distribution
over objectives precede the infimum over randomized algorithms.
Consequently, for every fixed $K$ and target accuracy $\epsilon$, their
construction provides an oracle and a distribution over objectives
that are chosen independently of the algorithm.

Theorem~3 gives the lower bounds
\[
\Omega(\Delta L\sigma^2\epsilon^{-4})
\]
in the bounded-variance regime and
\[
\Omega\!\left(
    \Delta\overline L\sigma\epsilon^{-3}
    +
    \sigma^2\epsilon^{-2}
\right)
\]
in the mean-square-smooth regime, for every fixed number $K$ of
simultaneous queries. Fixing positive numerical values of
$\Delta$, $L$, $\sigma$, and $\overline L$ gives the stated powers of
$\epsilon$, with constants independent of $\epsilon$.

Proposition~1 of~\citet{arjevani2023lower} shows that Theorem~3 remains
valid when the oracle is restricted to statistical-learning instances
satisfying
\[
g(\vtheta,Z)=\nabla f(\vtheta,Z),
\qquad
J(\vtheta)=\mathbb E_Z[f(\vtheta,Z)].
\]
Their results are stated for minimization; replacing both the
population and sample objectives by their negatives gives the
maximization formulation above without changing any gradient norm,
variance, or mean-square-smoothness constant.
\end{proof}

%%%
\subsection{Proofs of the Gaussian realization results}
\label{app:gaussian-realization}

\begin{proof}[Proof of Proposition~\ref{prop:gaussian-realization}]
Let $\mathcal A=\Reals$ and
\[
\mathcal S
=
\bigl(\{\mathsf{act}\}\times\mathcal Z\bigr)
\cup
\{s_{\mathsf{pay}},s_{\mathsf{abs}}\}.
\]
The initial state is $S_0=(\mathsf{act},Z)$, where $Z\sim\mathsf P$.
Define
\[
\pi_{\vtheta}(\cdot\mid s)
=
\begin{cases}
\mathcal N(f(\vtheta,z),1),
    & s=(\mathsf{act},z),\\
\mathcal N(0,1),
    & s\in\{s_{\mathsf{pay}},s_{\mathsf{abs}}\}.
\end{cases}
\]
The transitions are deterministic:
$(\mathsf{act},z)\to s_{\mathsf{pay}}\to s_{\mathsf{abs}}$, with
$s_{\mathsf{abs}}$ absorbing. Define
\[
r_{\vtheta}(s,a)
=
\begin{cases}
a-f(\vtheta,z),
    & s=(\mathsf{act},z),\\
F(\vtheta)/\gamma,
    & s=s_{\mathsf{pay}},\\
0,
    & s=s_{\mathsf{abs}}.
\end{cases}
\]

Writing $A_0=f(\vtheta,Z)+\xi$ with
$\xi\sim\mathcal N(0,1)$ gives $R_1=\xi$ and
$R_2=F(\vtheta)/\gamma$. Hence,
\[
\mathbb E_{\vtheta}\!\left[
    \sum_{h=0}^{\infty}\gamma^hR_{h+1}
\right]
=
\mathbb E\!\left[
    \xi+\gamma\frac{F(\vtheta)}{\gamma}
\right]
=
F(\vtheta).
\]

At the active state,
$S_{\vtheta,0}=\xi\nabla f(\vtheta,Z)$. The dummy policies are
parameter-independent, so $S_{\vtheta,h}=0$ for every $h\geq1$.
Therefore, the observable record has the stated form. Since
$\mathbb P(\xi=0)=0$, the decoder returns
$\nabla f(\vtheta,Z)$ almost surely.
\end{proof}

\begin{proof}[Proof of Proposition~\ref{prop:coupled-gaussian-realization}]
For a coupled query $(\vtheta,\vtheta')$, use the same $Z$ and $\xi$ in
the two active states. The resulting records are
\begin{align*}
\mathcal Y_1
=
\left(
    (\xi,\xi\nabla f(\vtheta,Z)),
    (F(\vtheta)/\gamma,0),
    (0,0),(0,0),\ldots
\right),
\\
\mathcal Y_2
=
\left(
    (\xi,\xi\nabla f(\vtheta',Z)),
    (F(\vtheta')/\gamma,0),
    (0,0),(0,0),\ldots
\right).
\end{align*}
Their decoded samples are $\nabla f(\vtheta,Z)$ and
$\nabla f(\vtheta',Z)$ almost surely. The result follows from the
assumed mean-square-smoothness inequality.
\end{proof}

\begin{proof}[Proof of Corollary~\ref{cor:gaussian-oracle-parameters}]
The decoded one-policy sample is $\nabla f(\vtheta,Z)$ almost surely.
By assumption, its expectation is $\nabla F(\vtheta)$, and its
mean-square error is at most $\sigma^2$. Hence it satisfies
Assumption~\ref{asm:bv-oracle}. The coupled samples are unbiased by the
same assumption, while their mean-square difference is controlled by
Proposition~\ref{prop:coupled-gaussian-realization}.
\end{proof}
%%%

The source lower bound in
Proposition~\ref{prop:arjevani-specialization} already allows the
algorithm to observe $J(\vtheta)$ exactly. In our simulation, this value
appears as the payment-state reward $J(\vtheta)/\gamma$.

\subsection{Simulation of complete reward-score feedback}
\label{app:feedback-simulation}

The next lemma makes the reduction explicit at the level of the complete
reward-score record, including the information contained separately in
the rewards and policy scores.

\begin{lemma}[Reward-score simulation]
\label{lem:reward-score-simulation}
Let $J(\vtheta)=\mathbb E_Z[f(\vtheta,Z)]$ and
$g(\vtheta,Z)=\nabla f(\vtheta,Z)$.
Given a source response $(J(\vtheta),g(\vtheta,Z))$, draw
$\xi\sim\mathcal N(0,1)$ independently and return
\[
\mathcal Y(\vtheta;Z,\xi)
=
\left(
    (\xi,\xi g(\vtheta,Z)),
    (J(\vtheta)/\gamma,0),
    (0,0),(0,0),\ldots
\right).
\]
Given the shared-seed responses at $\vtheta$ and $\vtheta'$, use the
same $\xi$ and return
\begin{align*}
\mathcal Y_1
=
\left(
    (\xi,\xi g(\vtheta,Z)),
    (J(\vtheta)/\gamma,0),
    (0,0),(0,0),\ldots
\right),
\\
\mathcal Y_2
=
\left(
    (\xi,\xi g(\vtheta',Z)),
    (J(\vtheta')/\gamma,0),
    (0,0),(0,0),\ldots
\right).
\end{align*}
These are exact simulations of the complete reward-score records in
Propositions~\ref{prop:gaussian-realization}
and~\ref{prop:coupled-gaussian-realization}.
\end{lemma}

\begin{proof}
Introduce the active action $A_0=f(\vtheta,Z)+\xi$. Its score is
$\xi g(\vtheta,Z)$ and its reward is $R_1=\xi$. At the payment state, the policy score is zero and the reward is
$R_2=J(\vtheta)/\gamma$. At every subsequent step, both the reward and
the policy score are zero.

Conversely, the source response is recovered almost surely as
$J(\vtheta)=\gamma R_2$ and
$g(\vtheta,Z)=S_{\vtheta,0}/R_1$. The coupled case is identical at each
endpoint.
\end{proof}

Under the decoder specified in
Proposition~\ref{prop:gaussian-realization}, the one-policy sample associated with Lemma~\ref{lem:reward-score-simulation} is
$\widehat g(\vtheta;Z,\xi)=g(\vtheta,Z)$ almost surely.
Consequently, the instances in
Proposition~\ref{prop:arjevani-specialization} satisfy
Assumption~\ref{asm:bv-oracle} with
$\nu_0^2=\sigma_0^2$ and, in the mean-square-smooth regime,
Assumption~\ref{asm:mss-oracle} with
$L_{\mathrm{pair},0}^2=\overline L_0^2$.
These constants are independent of $\epsilon$.

\subsection{Proof of Theorem~\ref{th:lower1}}

\badlower*

\begin{proof}
Set $c_{\mathrm{BV}}\coloneqq\overline c_{\mathrm{BV}}/2$ and decrease
$\epsilon_0$, if necessary, so that
$1\leq c_{\mathrm{BV}}\epsilon^{-4}$ for every
$\epsilon\in(0,\epsilon_0)$. Fix $\epsilon\in(0,\epsilon_0)$ and denote
by $\mathsf P_{\epsilon}^{\mathrm{BV}}$ the distribution over
bounded-variance stochastic-optimization instances in the first part of
Proposition~\ref{prop:arjevani-specialization}. For every instance in
its support, apply Proposition~\ref{prop:gaussian-realization} with the
corresponding sample function $f$. Let
$\mathsf Q_{\epsilon}^{\mathrm{BV}}$ be the resulting pushforward
distribution over black-box policy-optimization instances.

The smoothness and initial-gap claims hold for every instance in the
support by Proposition~\ref{prop:arjevani-specialization}. The
bounded-variance oracle claim follows from
Corollary~\ref{cor:gaussian-oracle-parameters}, with
$\nu_0^2=\sigma_0^2$.

Suppose, toward a contradiction, that a randomized adaptive algorithm
$\mathfrak A$ operating under Definition~\ref{def:one-policy-feedback}
uses fewer than $c_{\mathrm{BV}}\epsilon^{-4}$ calls and satisfies
\[
\mathbb E_{\mathcal I\sim\mathsf Q_{\epsilon}^{\mathrm{BV}}}
\left[
    \left\|
        \nabla J_{\mathcal I}(\vtheta_{\mathrm{OUT}})
    \right\|
\right]
\leq\epsilon.
\]

Construct a stochastic-optimization algorithm $\mathfrak B$ as follows. Whenever
$\mathfrak A$ queries $\vtheta_t$, algorithm $\mathfrak B$ queries the source oracle at the same parameter and obtains
$(J(\vtheta_t),g(\vtheta_t,Z_t))$. It independently draws $\xi_t\sim\mathcal N(0,1)$, and returns to $\mathfrak A$ the record
\[
\left(
    (\xi_t,\xi_t g(\vtheta_t,Z_t)),
    (J(\vtheta_t)/\gamma,0),
    (0,0),(0,0),\ldots
\right).
\]
By Lemma~\ref{lem:reward-score-simulation}, the conditional distribution of every response is exactly the one prescribed by Definition~\ref{def:one-policy-feedback}. Hence the complete interaction history and the output of $\mathfrak A$ have the correct joint law.

If necessary, $\mathfrak B$ makes one final query at
$\vtheta_{\mathrm{OUT}}$ so that the output is one of its queried
points, as required by the oracle protocol of
\citet{arjevani2023lower}. Therefore, when the source instance is drawn
from $\mathsf P_{\epsilon}^{\mathrm{BV}}$, an algorithm using $T$
black-box calls induces a stochastic-optimization algorithm using at
most $T+1$ oracle rounds and having exactly the same expected gradient
norm.

Since $T<c_{\mathrm{BV}}\epsilon^{-4}$ and
$1\leq c_{\mathrm{BV}}\epsilon^{-4}$,
\[
T+1
<
2c_{\mathrm{BV}}\epsilon^{-4}
=
\overline c_{\mathrm{BV}}\epsilon^{-4}.
\]
This contradicts the first distributional lower bound in
Proposition~\ref{prop:arjevani-specialization}.
\end{proof}

\subsection{Proof of Theorem~\ref{th:lower2}}

\goodlower*

\begin{proof}
Set $c_{\mathrm{MSS}}\coloneqq\overline c_{\mathrm{MSS}}/2$ and decrease
$\epsilon_0$, if necessary, so that
$1\leq c_{\mathrm{MSS}}\epsilon^{-3}$ for every
$\epsilon\in(0,\epsilon_0)$. Fix $\epsilon\in(0,\epsilon_0)$ and denote
by $\mathsf P_{\epsilon}^{\mathrm{MSS}}$ the distribution over
mean-square-smooth stochastic-optimization instances in the second part
of Proposition~\ref{prop:arjevani-specialization}. For every instance
in its support, apply
Propositions~\ref{prop:gaussian-realization}
and~\ref{prop:coupled-gaussian-realization} with the corresponding
sample function $f$, and the resulting discounted performance function is $J$ by
Proposition~\ref{prop:gaussian-realization}. Let
$\mathsf Q_{\epsilon}^{\mathrm{MSS}}$ be the resulting pushforward
distribution over black-box policy-optimization instances.

Proposition~\ref{prop:arjevani-specialization} and
Corollary~\ref{cor:gaussian-oracle-parameters} verify
Assumptions~\ref{asm:smooth-performance} and~\ref{asm:bv-oracle} for
every instance in the support.
Proposition~\ref{prop:coupled-gaussian-realization} verifies
Assumption~\ref{asm:mss-oracle}, with
$L_{\mathrm{pair},0}^2=\overline L_0^2$.

Suppose, toward a contradiction, that a randomized adaptive algorithm
$\mathfrak A$ operating under Definition~\ref{def:coupled-feedback}
uses fewer than $c_{\mathrm{MSS}}\epsilon^{-3}$ black-box calls and satisfies
\[
\mathbb E_{\mathcal I\sim\mathsf Q_{\epsilon}^{\mathrm{MSS}}}
\left[
    \left\|
        \nabla J_{\mathcal I}(\vtheta_{\mathrm{OUT}})
    \right\|
\right]
\leq\epsilon.
\]

Construct an algorithm $\mathfrak B$ for the $K=2$ shared-seed oracle
of~\citet{arjevani2023lower}. If $\mathfrak A$ makes a coupled query
$(\vtheta_t,\vtheta_t')$, algorithm $\mathfrak B$ queries the batch
$(\vtheta_t,\vtheta_t')$. It receives
\[
\left(
    J(\vtheta_t),g(\vtheta_t,Z_t)
\right),
\qquad
\left(
    J(\vtheta_t'),g(\vtheta_t',Z_t)
\right),
\]
draws one independent $\xi_t\sim\mathcal N(0,1)$, and returns
\begin{align*}
\left(
    (\xi_t,\xi_t g(\vtheta_t,Z_t)),
    (J(\vtheta_t)/\gamma,0),
    (0,0),(0,0),\ldots
\right),
\\
\left(
    (\xi_t,\xi_t g(\vtheta_t',Z_t)),
    (J(\vtheta_t')/\gamma,0),
    (0,0),(0,0),\ldots
\right).
\end{align*}
If $\mathfrak A$ makes a one-policy query at $\vtheta_t$, algorithm
$\mathfrak B$ queries the repeated batch
$(\vtheta_t,\vtheta_t)$ and returns the first simulated record.

By Lemma~\ref{lem:reward-score-simulation}, this procedure exactly simulates every complete reward-score record observed by $\mathfrak A$. Moreover, one black-box call is simulated by one $K=2$ oracle round. As in the proof of Theorem~\ref{th:lower1}, at most one additional round is needed to query the final output.

Since $T<c_{\mathrm{MSS}}\epsilon^{-3}$ and
$1\leq c_{\mathrm{MSS}}\epsilon^{-3}$,
\[
T+1
<
2c_{\mathrm{MSS}}\epsilon^{-3}
=
\overline c_{\mathrm{MSS}}\epsilon^{-3}.
\]
This contradicts the second distributional lower bound in
Proposition~\ref{prop:arjevani-specialization}.
\end{proof}

\FloatBarrier

% \section{Experimental details}\label{app:exp}
% In this section, we provide pseudocode for the variants of \algo
% based on other variance-reduced policy-gradient algorithms. The
% hyperparameters used in the CartPole experiments are reported in
% Table~\ref{tab:cart}.

% \begin{table}[t]
% \caption{Hyperparameters for CartPole. We denote by $p$ the refresh
% probability of PAGE-PG and DEF-PG, and by $\beta$ the momentum
% coefficient of STORM-PG.}\label{tab:cart}
% \begin{tabular}{@{}lcr@{}}
% \toprule
% Algorithm      & Step Size & Other Parameters \\ \midrule
% DEF-PG         & $5\cdot10^{-4}$          &   $p=0.8$              \\
% PAGE-PG         & $5\cdot10^{-5}$          &    $p=0.8$             \\
% STORM-PG        & $4\cdot10^{-5}$          &   $\beta=0.99$              \\
% DEF-PG(STORM)  & $4\cdot10^{-4}$          &   $\beta=0.99$              \\
% SRVR-PG         & $5\cdot10^{-7}$          &                 \\
% DEF-PG(SRVR-PG) & $10^{-6}$          &                 \\
% SVRPG          & $10^{-4}$          &                 \\
% DEF-PG(SVRPG)  & $1.5\cdot10^{-4}$          &                 \\ \bottomrule
% \end{tabular}

% \end{table}

\section{Additional Defensive Policy-Gradient Variants}
\label{app:variants}
In this section, we provide pseudocode for the variants of \algo
based on SVRPG, SRVR-PG, and STORM-PG.

\begin{algorithm}[h]
\caption{DEF-PG (SVRPG variant)}
\label{algo:def-svrpg}
\begin{algorithmic}[1]
\State \textbf{Input:} initial parameter $\vtheta_0$, number of epochs
$S$, epoch length $m$, large batch size $N$, small batch size $B$,
step size $\eta>0$, defensive parameter $\alpha\in(0,1)$
\State $\widetilde{\vtheta}^{\,0}\gets\vtheta_0$
\For{$s=0,\ldots,S-1$}
    \State $\vtheta_0^s\gets\widetilde{\vtheta}^{\,s}$
    \For{$i=1,\ldots,N$}
        \State Sample
        $\overline{\tau}_i\sim
        p(\cdot\mid\widetilde{\vtheta}^{\,s})$
    \EndFor
    \State $\mu_s\gets
    \frac{1}{N}\sum_{i=1}^N
    g(\widetilde{\vtheta}^{\,s},\overline{\tau}_i)$
    \State $v_0^s\gets\mu_s$
    \State $\vtheta_1^s\gets\vtheta_0^s+\eta v_0^s$
    \For{$t=1,\ldots,m-1$}
        \For{$i=1,\ldots,B$}
            \State Sample
            $\tau_i\sim
            q_{\alpha,\vtheta_t^s,\widetilde{\vtheta}^{\,s}}$
            as in Algorithm~\ref{algo_ppgdef}
            \State $\omega_{t,i}^{s,+}\gets
            \omega_\alpha
            (\vtheta_t^s,\widetilde{\vtheta}^{\,s};\tau_i)$
            \State $\omega_{t,i}^{s,-}\gets
            \omega_\alpha'
            (\vtheta_t^s,\widetilde{\vtheta}^{\,s};\tau_i)$
        \EndFor
        \State $v_t^s\gets
        \mu_s+
        \frac{1}{B}\sum_{i=1}^B
        \left[
            \omega_{t,i}^{s,+}g(\vtheta_t^s,\tau_i)
            -
            \omega_{t,i}^{s,-}
            g(\widetilde{\vtheta}^{\,s},\tau_i)
        \right]$
        \State $\vtheta_{t+1}^s\gets\vtheta_t^s+\eta v_t^s$
    \EndFor
    \State $\widetilde{\vtheta}^{\,s+1}\gets\vtheta_m^s$
\EndFor
\State \textbf{Output:} $\vtheta_{\mathrm{OUT}}$ chosen uniformly from
$\{\vtheta_t^s:
s=0,\ldots,S-1,\ t=0,\ldots,m-1\}$
\end{algorithmic}
\end{algorithm}

\begin{algorithm}[h]
\caption{DEF-PG (SRVR-PG variant)}
\label{algo:def-SRVR-PG}
\begin{algorithmic}[1]
\State \textbf{Input:} initial parameter $\vtheta_0$, number of epochs
$S$, epoch length $m$, large batch size $N$, small batch size $B$,
step size $\eta>0$, defensive parameter $\alpha\in(0,1)$
\State $\widetilde{\vtheta}^{\,0}\gets\vtheta_0$
\For{$s=0,\ldots,S-1$}
    \State $\vtheta_0^s\gets\widetilde{\vtheta}^{\,s}$
    \For{$i=1,\ldots,N$}
        \State Sample
        $\overline{\tau}_i\sim
        p(\cdot\mid\widetilde{\vtheta}^{\,s})$
    \EndFor
    \State $v_0^s\gets
    \frac{1}{N}\sum_{i=1}^N
    g(\widetilde{\vtheta}^{\,s},\overline{\tau}_i)$
    \State $\vtheta_1^s\gets\vtheta_0^s+\eta v_0^s$
    \For{$t=1,\ldots,m-1$}
        \For{$i=1,\ldots,B$}
            \State Sample
            $\tau_i\sim
            q_{\alpha,\vtheta_t^s,\vtheta_{t-1}^s}$
            as in Algorithm~\ref{algo_ppgdef}
            \State $\omega_{t,i}^{s,+}\gets
            \omega_\alpha
            (\vtheta_t^s,\vtheta_{t-1}^s;\tau_i)$
            \State $\omega_{t,i}^{s,-}\gets
            \omega_\alpha'
            (\vtheta_t^s,\vtheta_{t-1}^s;\tau_i)$
        \EndFor
        \State $v_t^s\gets
        v_{t-1}^s+
        \frac{1}{B}\sum_{i=1}^B
        \left[
            \omega_{t,i}^{s,+}g(\vtheta_t^s,\tau_i)
            -
            \omega_{t,i}^{s,-}g(\vtheta_{t-1}^s,\tau_i)
        \right]$
        \State $\vtheta_{t+1}^s\gets\vtheta_t^s+\eta v_t^s$
    \EndFor
    \State $\widetilde{\vtheta}^{\,s+1}\gets\vtheta_m^s$
\EndFor
\State \textbf{Output:} $\vtheta_{\mathrm{OUT}}$ chosen uniformly from
$\{\vtheta_t^s:
s=0,\ldots,S-1,\ t=0,\ldots,m-1\}$
\end{algorithmic}
\end{algorithm}

\begin{algorithm}[h]
\caption{DEF-PG (STORM-PG variant)}
\label{algo:def-STORM-PG}
\begin{algorithmic}[1]
\State \textbf{Input:} initial parameter $\vtheta_0$, initial batch
size $N_0$, mini-batch size $B$, step size $\eta>0$, defensive
parameter $\alpha\in(0,1)$, momentum parameter $\beta\in(0,1)$,
number of iterations $T$
\For{$i=1,\ldots,N_0$}
    \State Sample
    $\overline{\tau}_i\sim p(\cdot\mid\vtheta_0)$
\EndFor
\State $v_0\gets
\frac{1}{N_0}\sum_{i=1}^{N_0}
g(\vtheta_0,\overline{\tau}_i)$
\For{$t=0,\ldots,T-1$}
    \State $\vtheta_{t+1}\gets\vtheta_t+\eta v_t$
    \For{$i=1,\ldots,B$}
        \State Sample
        $\tau_i\sim q_{\alpha,\vtheta_{t+1},\vtheta_t}$
        as in Algorithm~\ref{algo_ppgdef}
        \State $\omega_{t+1,i}^{+}\gets
        \omega_\alpha(\vtheta_{t+1},\vtheta_t;\tau_i)$
        \State $\omega_{t,i}^{-}\gets
        \omega_\alpha'(\vtheta_{t+1},\vtheta_t;\tau_i)$
    \EndFor
    \State $v_{t+1}\gets
    (1-\beta)v_t+
    \frac{1}{B}\sum_{i=1}^B
    \left[
        \omega_{t+1,i}^{+}g(\vtheta_{t+1},\tau_i)
        -
        (1-\beta)\omega_{t,i}^{-}g(\vtheta_t,\tau_i)
    \right]$
\EndFor
\State \textbf{Output:} $\vtheta_{\mathrm{OUT}}$ chosen uniformly from
$\{\vtheta_t\}_{t=0}^{T-1}$
\end{algorithmic}
\end{algorithm}